\documentclass{article}

\usepackage{microtype}
\usepackage{graphicx}
\usepackage{subcaption}
\usepackage{booktabs} 

\usepackage{hyperref}
\usepackage{paralist}

\usepackage{stfloats}
\usepackage{multirow}
\usepackage{longtable}

\usepackage[preprint]{icml2026}

\usepackage{amsmath}
\usepackage{amssymb}
\usepackage{mathtools}
\usepackage{amsthm}

\usepackage[capitalize,noabbrev]{cleveref}

\theoremstyle{plain}
\newtheorem{theorem}{Theorem}[section]

\newtheorem{lemma}[theorem]{Lemma}

\theoremstyle{definition}
\newtheorem{definition}[theorem]{Definition}

\theoremstyle{remark}

\usepackage{algorithmic}
\renewcommand{\algorithmicrequire}{\textbf{Input:}}
\renewcommand{\algorithmicensure}{\textbf{Output:}}
\renewcommand{\algorithmiccomment}[1]{\hfill $\triangleright$ #1}
\newcommand{\CommentLeft}[1]{\(\triangleright\) #1}

\usepackage[textsize=tiny]{todonotes}

\icmltitlerunning{PAMod: Modeling Cyclical Shifts via Phase-Amplitude Modulation for Non-stationary Time Series Forecasting}

\begin{document}

\twocolumn[
  \icmltitle{PAMod: Modeling Cyclical Shifts via Phase-Amplitude Modulation for Non-stationary Time Series Forecasting}



  \icmlsetsymbol{equal}{*}

  \begin{icmlauthorlist}
    \icmlauthor{Yingbo Zhou}{1}
    \icmlauthor{Yutong Ye}{2}
    \icmlauthor{Shuhao Li}{1}
    \icmlauthor{Rui Qian}{1}
    \icmlauthor{Qiang Huang}{1}
    \icmlauthor{Lemao Liu}{1}
    \icmlauthor{Li Sun}{3}
    \icmlauthor{Dejing Dou}{1}
  \end{icmlauthorlist}

  \icmlaffiliation{1}{Fudan University}
  \icmlaffiliation{2}{Beihang University}
  \icmlaffiliation{3}{Beijing University of Posts and Telecommunications}

  \icmlcorrespondingauthor{Dejing Dou}{doudejing@fudan.edu.cn}

  \icmlkeywords{Machine Learning, ICML}

  \vskip 0.3in
]



\printAffiliationsAndNotice{}  

\begin{abstract}
  Real-world time series forecasting faces the fundamental challenge of non-stationary statistical properties, including shifts in mean and variance over time. 
  While reversible instance normalization (RevIN) has shown promise by stationarizing inputs and denormalizing outputs, it relies on the strong assumption that historical and future distributions remain identical. 
  We observe that in many practical applications, distribution shifts follow cyclical patterns that correlate with periodic positions (e.g., seasonal and holiday volatility). 
  To this end, we propose \textbf{PAMod}, a lightweight yet powerful framework that models cyclical distribution shifts via \textbf{P}hase-\textbf{A}mplitude \textbf{Mod}ulation in the normalized feature space. 
  PAMod learns periodic embeddings to modulate representations: phase modulation captures mean shifts, while amplitude modulation adapts to variance changes. 
  Crucially, we prove mathematically that modulating in normalized space is equivalent to applying dynamic denormalization, offering an elegant unification of distribution adaptation and representation learning. 
  Extensive experiments on twelve real-world benchmarks demonstrate that PAMod achieves state-of-the-art performance with fewer computational resources.
  Furthermore, our modulation mechanism, as a novel plug-and-play technique, can improve existing time-series forecasting methods with simple integration.
\end{abstract}

\section{Introduction}

Time series forecasting underpins critical decision-making across domains, including energy management \cite{akay2007grey, DBLP:journals/tkdd/FanFZBZX24}, finance analysis \cite{DBLP:journals/asc/SezerGO20, DBLP:conf/www/Lin0Q0CZLG25}, healthcare monitoring \cite{bertozzi2020challenges, DBLP:conf/nips/WangHWZ23}, and traffic optimization \cite{DBLP:journals/tits/ShuCX22, DBLP:conf/icde/00010GY0HXJ24}. 
Despite significant advances in deep learning architectures, real-world time series remain particularly challenging due to their inherent non-stationary statistical properties (mean, variance, and higher moments) evolve \cite{DBLP:conf/iclr/KimKTPCC22, DBLP:conf/icml/LiuWHL0BX25}, violating the stationarity assumptions of many models.

Empirical evidence of the non-stationary is shown in Figure \ref{fig:1}(a), where the training and test sets of real-world time series (e.g., the OT channel in ETTh1) exhibit distinct probability distributions with diverging means and variances.
The dashed lines marking the respective means illustrate that statistical properties of the test set often systematically diverge from those of the training set.
To bridge this gap, two dominant research paradigms have emerged.
The first, the normalization paradigm \cite{DBLP:conf/iclr/KimKTPCC22, DBLP:conf/nips/LiuCLHLXC23, DBLP:conf/nips/0001WLLYXZ24}, seeks to eliminate non-stationarity by coercing series into a stationary space for modeling. 
The second paradigm relies on complex architectures \cite{DBLP:conf/nips/LiuWWL22, DBLP:conf/iclr/WangDAY23,  DBLP:conf/icml/LiuWHL0BX25}, which attempt to implicitly model and absorb the full spectrum of distributional shifts within their vast parameter space.

We identify a critical, shared limitation in both approaches: they fail to explicitly model the structured, predictable component of non-stationarity that arises from cyclical dependencies. Normalization-based methods discard this structure by design, while complex architectures inefficiently attempt to rediscover it from data, often conflating it with noise. 
This leads to two key shortcomings: (1) \textbf{suboptimal parameter efficiency}, as either useful structure is removed or must be laboriously re-learned, and (2) \textbf{vulnerable generalization}, particularly when predictable cyclical shifts dominate, as models lack an inductive bias to capture them robustly.

\begin{figure*}[ht]
\centering
\includegraphics[width=\textwidth]{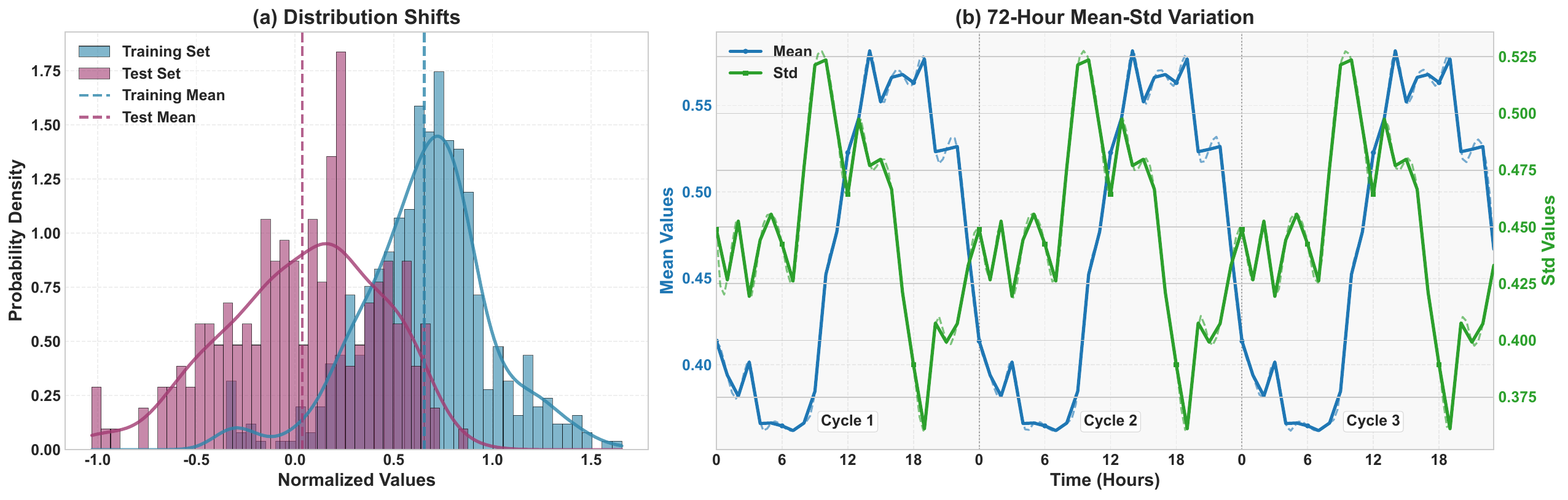}
\caption{
Visualization of distribution shift and periodic mean-std variation in the OT channel of ETTh1.
(a) histograms and density curves illustrate distinct probability distributions, with dashed lines marking the respective means in training and test sets. (b) 72-hour cyclic variations exhibit stable 24-hour repetition across three consecutive cycles (Cycle 1–3), highlighting inherent diurnal periodicity. Notably, distribution shifts originate from decoupled mean and variance variations within these cycles.}
\label{fig:1}
\vspace{-0.15cm}
\end{figure*}

We argue that rather than indiscriminately eliminating or preserving non-stationarity, we should distinguish between random fluctuations and structured cyclical shifts. 
As seen in Figure \ref{fig:1}(b), over 72 consecutive hours (three 24-hour cycles), the OT channel in ETTh1 exhibits stable diurnal periodicity, where mean variations and variance fluctuations maintain cyclical regularity.
Crucially, these components evolve independently yet synchronously across cycles, demonstrating that distribution shifts originate from structured rather than arbitrary variations.
This observation motivates a paradigm shift: \textbf{distribution shifts in time series can decompose into cyclical components with correlated positions}.
By modeling these cyclical patterns explicitly, we can extend the normalization paradigm to adapt to cyclical non-stationarity while capturing significant distribution shifts with a simple design.

To explicitly model the cyclical shifts, we propose \textbf{P}hase-\textbf{A}mplitude \textbf{Mod}ulation (\textbf{PAMod}), a novel framework that learns cyclical distribution shifts through periodic modulation design.
Inspired by the phase-amplitude modulation in communication systems \cite{roder2006amplitude}, PAMod introduces two complementary mechanisms: phase modulation captures systematic mean shifts via additive transformations, while amplitude modulation adapts to variance changes via multiplicative scaling.
These mechanisms operate in the normalized feature space, achieving equivalent distribution adaptation through a more elegant and unified approach.

In summary, our main contributions are three folds:
\begin{compactitem}
  \item We formalize cyclical distribution shifts as mean-variance decoupled variations, providing new analytical insight into periodic non-stationarity.
  \item We propose PAMod with the phase-amplitude modulation mechanism that learns to adjust means (phase) and variances (amplitude) based on cyclical positions.
  \item We establish new state-of-the-art performance across multiple real-world benchmarks, reducing MSE by 4–16\% with superior memory efficiency (5–10× fewer parameters than transformer-based methods).
\end{compactitem}

\section{Related Work}

\textbf{Time Series Forecasting Architectures.}
Deep learning has revolutionized time series forecasting through diverse architectures. 
Recurrent Neural Networks (RNNs) model sequential dependencies through recurrent states \cite{DBLP:conf/sigir/LaiCYL18, DBLP:journals/corr/abs-2308-11200}, while Temporal Convolutional Networks (TCNs) employ dilated convolutions for efficient long-range modeling \cite{DBLP:conf/nips/LiuZCXLM022, DBLP:conf/iclr/LuoW24}. 
Transformer-based approaches \cite{DBLP:conf/aaai/ZhouZPZLXZ21, DBLP:conf/icml/ZhouMWW0022, DBLP:conf/kdd/Piao0MMS24} leverage self-attention or frequency-domain mechanisms, and MLP-based methods \cite{DBLP:conf/aaai/ZengCZ023, DBLP:conf/nips/YiZFWWHALCN23} demonstrate competitive performance with minimal complexity. 
Graph Neural Networks (GNNs) capture multivariate correlations via spatiotemporal graphs \cite{DBLP:conf/nips/YiZFHHWACN23, DBLP:conf/aaai/CaiLLFW24}, and state space models offer linear-time sequence modeling via selective scanning \cite{DBLP:journals/corr/abs-2312-00752}. 
Despite their advances, these architectures often assume stationarity or treat periodicity as the auxiliary factor, leaving a gap for designs that explicitly unify periodic modeling with distribution shift mitigation.

\textbf{Handing Non-stationary in Time Series.}
Non-stationarity is a core challenge in time series forecasting. 
Classical statistical methods \cite{box2015time} address this by applying temporal differencing to enforce stationarity.
Recently, normalization techniques \cite{DBLP:conf/iclr/KimKTPCC22, DBLP:conf/nips/LiuCLHLXC23, DBLP:conf/nips/0001WLLYXZ24} have emerged as the de facto standardfor eliminating non-stationarity. 
Conversely, some specific designs for certain network architectures \cite{DBLP:conf/nips/LiuWWL22, DBLP:conf/iclr/WangDAY23,  DBLP:conf/icml/LiuWHL0BX25} have been introduced to capture distribution shifts. 
Additionally, decomposition-based methods  \cite{DBLP:conf/iclr/WuHLZ0L23, DBLP:conf/iclr/WangWSHLMZ024} separate time series into trend and seasonal components to mitigate non-stationary effects. 
A key limitation of these approaches is that they treat non-stationarity as a monolithic disturbance to be eliminated, rather than structured and periodic variations that can be explicitly modeled for targeted correction.

\textbf{Periodic Modeling in Time Series.}
Periodicity is a ubiquitous property of real-world time series (e.g., diurnal cycles in sensor data), and modeling it effectively can enhance forecasting performance.
Classical methods \cite{cleveland1990stl, box2015time} explicitly capture periodic patterns via predefined functions. 
To better leverage periodic information, current approaches \cite{DBLP:conf/iclr/0001WLLB0X24, DBLP:conf/icml/Lin0WCY24, DBLP:conf/nips/Lin0HWMZ24} primarily focus on explicitly identifying, representing, or transforming repeating temporal patterns in the data.
For instance, CycleNet \cite{DBLP:conf/nips/Lin0HWMZ24} introduces learnable recurrent cycles to model periodic patterns, while our concurrent work \cite{ConcurrentWork1}, disentangles periodic patterns as phase and amplitude components. 
Crucially, existing periodic modeling works rarely connect periodicity to distribution shifts, which treat periodic patterns as predictive features rather than the root cause of structured drift.

Unlike the above works, PAMod bridges these gaps through a cohesive design: 1) a lightweight architecture that avoids parameter inefficiency; 2) an explicit modeling approach to non-stationarity that preserves informative cyclical structures; and 3) a novel mechanism that directly ties periodicity to distribution shifts via the phase-amplitude modulation. 
To the best of our knowledge, PAMod pioneers the use of \textbf{explicit cyclical modeling to directly adapt to the non-stationary nature of time series}.

\section{Methodology}

\subsection{Preliminaries}

\textbf{Problem Definition}.
Given historical time series $X=\{x_1, x_2, \cdots, x_T\}\in\mathbb{R}^{T\times C}$ with $C$ variates, the goal of time series forecasting is to predict next $H$-step values $\hat{Y}=\{\hat{x}_{T+1}, \hat{x}_{T+2}, \cdots, \hat{x}_{T+H}\}\in\mathbb{R}^{H\times C}$ with learnable forecasting function $f_{\theta}(\cdot)$:
\begin{equation}
    \hat{x}_{T+1:T+H}=f_{\theta}(x_{1:T}).
\end{equation}

\textbf{RevIN Formulation}. 
To mitigate the non-stationarity of time series, Reversible Instance Normalization (RevIN), as the normalization paradigm, is adopted by almost all the mainstream forecasters.
Concretely, RevIN \cite{DBLP:conf/iclr/KimKTPCC22} normalizes the inputs as follows:
\begin{equation}
\begin{aligned}
        \overline{x}_{t}=&\frac{x_t-\mu_{X}}{\sigma_{X}}, ~~~\mu_{X}=\frac{1}{T}\sum_{t=1}^{T}x_t, \\
        \sigma_{X}&=\sqrt{\frac{1}{T}\sum_{t=1}^{T}(x_t-\mu_{X})^2}.
\end{aligned}
\end{equation}
After forecasting $\overline{y}_m=f_{\theta}(\overline{x}_t)$, it denormalizes:
\begin{equation}
    \hat{y}_{m}=\overline{y}_{m}\cdot\sigma_{X} + \mu_{X},
    \label{eq:3}
\end{equation}
where $m\in\{T+1, \cdots, T+H\}$.
This assumes consistent distributions between historical inputs $X$ and future ground truth $Y$, i.e., $\mu_{X}=\mu_{Y}, \sigma_{X}=\sigma_{Y}$, which fails in real-world applications as distribution shifts occur.

\begin{figure}[ht]
\centering
\includegraphics[width=\linewidth]{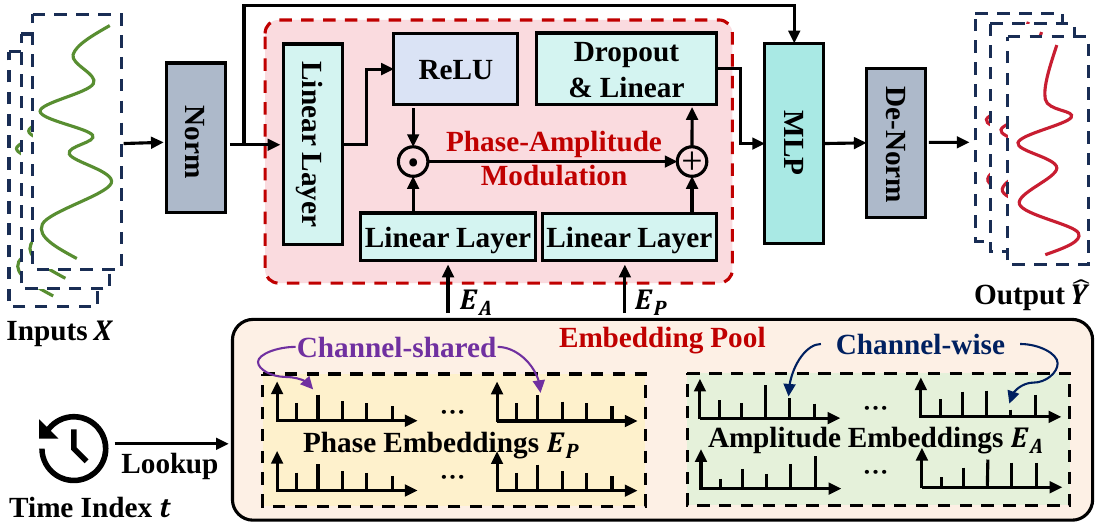}
\caption{Overview of PAMod, which consists of normalization, phase-amplitude modulation with the embedding pool, MLP and denormalization for non-stationary time series forecasting.}
\label{fig:2}
\vspace{-0.15cm}
\end{figure}

\subsection{Overview of PAMod}
In Figure \ref{fig:1}(b), we can observe that distribution shifts can be decomposed into mean shifts (phase) and variance changes (amplitude), both functions of periodic positions.
To this end, we model cyclical shifts via phase-amplitude modulation in the normalized feature space.
Our proposed method, PAMod, consists of instance-wise normalization, learnable cyclical embedding-guided phase-amplitude modulation, MLP and denormalization, illustrated in Figure \ref{fig:2}.
While normalization removes arbitrary, non-cyclical distribution shifts to stabilize learning, our approach explicitly models and reintroduces the structured, predictable non-stationarity inherent in cyclical patterns to enhance forecasting accuracy.

\subsection{Core Components of PAMod}
\textbf{Intuition}.
The design of PAMod is based on the generalized expression of amplitude-phase modulation in communication theory \cite{bloch1944modulation, roder2006amplitude}.

Consider a carrier wave $\cos{(2\pi f_c t)}$ whose amplitude $A(t)$ and phase $P(t)$ are simultaneously modulated by information signals $M_A(t)$ and $M_P(t)$:
\begin{equation}
    x_{M}(t)=\underbrace{[A_c + M_{A}(t)]}_{A(t)}\cdot \cos(2\pi f_c t+\underbrace{k\cdot M_{P}(t)}_{P(t)}),
\end{equation}
where $A_c$ is the constant carrier amplitude, and $k$ is the phase modulation index. 
To elucidate the combined effect, we apply a first-order Taylor expansion  under the assumption of a small phase deviation (i.e., $kM_P(t)\ll 1$), which corresponds to modeling smooth distribution shifts:
\begin{equation}
\begin{aligned}
    x(t)&\approx A(t)[\cos(2\pi f_ct)-kM_P(t)\sin(2\pi f_ct)] \\
    &=\underbrace{A(t)\cos(2\pi f_c t)}_{\text{In-phase Component}} - \underbrace{A(t)kM_P(t)\sin(2\pi f_c t)}_{\text{Quadrature Component}}.
\end{aligned}
\end{equation}
Since $kM_P(t)\ll 1$, the magnitude of the quadrature term satisfies $A(t)kM_P(t)\ll A(t)$.
Consequently, the quadrature term can be treated as a weak additive perturbation $\Phi (t)$, and the primary signal satisfies:
\begin{equation}
    x(t)\approx A(t)\cos(2\pi f_c t) + \Phi(t).
\end{equation}
Recalling Equation (\ref{eq:3}), we can find that the denormalization process and the phase-amplitude modulation share a perfectly aligned mathematical structure.
Therefore, instead of using static scaling factors, we can modulate the normalized features with learnable cyclical signals to guide the learning of dynamic distribution shifts.

\textbf{Learnable Cyclical Embedding}.
To obtain modulation signals that resemble $A(t)$ and $\Phi(t)$ within the time series modeling framework, we propose learnable cyclical embeddings, which serve as data-driven proxies that parameterize the structured phase (mean offset) and amplitude (variance fluctuation) variations.
 
Considering that real-world time series typically exhibit prominent and explicit cyclic patterns \cite{DBLP:conf/nips/Lin0HWMZ24, DBLP:conf/icml/LinCWQL25}, the specific cycle length $L$ is available and straightforward via autocorrelation functions \cite{madsen2007time}.
Given a learnable mean matrix $W_m\in\mathbb{R}^{L\times T}$, a shared embedding vector is retrieved based on the cyclical position of the absolute time $t$ for all the channels:
\begin{equation}
    E_P(t)=\text{Lookup}(W_m, t~\text{mod}~L)\in \mathbb{R}^{1\times T},
\end{equation}
where the Lookup operation conditions the model on the absolute phase within the cycle ($t~\text{mod}~L$), enabling it to dynamically adapt to systematic mean offsets that correlate with periodic patterns.

Since the heterogeneous variance shifts observed across different channels, we define a channel-wise embedding tensor $W_v\in\mathbb{R}^{L\times C \times T}$  to capture cycle-dependent variance fluctuations for each channel $i\in [0, C-1]$:
\begin{equation}
    E_A^i(t) =\text{Lookup}(W_v, i, t~\text{mod}~L) \in \mathbb{R}^{C\times T}.
\end{equation}
The amplitude embedding $E_a^i(t)$ allows each variate to independently learn to amplify or suppress variance fluctuations based on its position within the periodic cycle.

\textbf{Phase-Amplitude Modulation}.
With learnable $E_P(t)$ and $E_A(t)$, we implement our phase-amplitude modulation to transform the normalized features $X_{\text{norm}}$. 
First, we use linear projection $U_1$ with ReLU activation to extract nonlinear temporal patterns as carrier waves:
\begin{equation}
    X_1 = \text{ReLU}(U_1X_{\text{norm}}), ~~U_1\in\mathbb{R}^{T\times S},
\end{equation}
where $S=4T$ in the hidden dimension.
After that, we project the learnable embeddings to bridge the gap between the embedding and feature space, and then model cyclical shifts via element-wise multiplication and additive fusion:
\begin{equation}
    X_2 = X_1\odot U_2E_A(t)+U_3E_P(t), ~~U_2, U_3\in\mathbb{R}^{T\times S}.
\end{equation}
Finally, the modulated features $X_{\text{mod}}$ can be obtained with linear projection and Dropout operation:
\begin{equation}
    X_{\text{mod}}=\text{Dropout}(U_4X_2), ~~U_4\in\mathbb{R}^{S\times T}.
\end{equation}

\textbf{MLP-based Prediction Head}.
We adopt a two-layer MLP architecture, optionally with Dropout operation, to project the modulated features into the expected prediction horizon:
\begin{equation}
    \overline{Y}=U_6(\text{Dropout}(\text{GeLU}(U_5(X_{\text{norm}}+X_{\text{mod}}))),
    \label{eq:12}
\end{equation}
where $U_5\in \mathbb{R}^{d\times T}$ and $U_6\in\mathbb{R}^{H\times d}$.

\subsection{Theoretical Analysis}
We provide a concise theoretical argument explaining why explicitly modeling cycle-dependent distribution shift via phase-amplitude modulation improves generalization, especially under periodic non-stationarity.

\paragraph{Periodic distribution shift setting.}
Let $X_t\in\mathbb{R}^{C\times T}$ be an input window ending at absolute time $t$, and let $\tau=t\bmod L$ denote the cyclical position. Real-world series (e.g., load/traffic) often exhibit \emph{periodic distribution shift} $\mathcal{D}$: the conditional distribution changes with $\tau$ while remaining relatively stable across different cycles:
\begin{equation}
\mathcal{D}_t \equiv \mathcal{D}_{\tau}, \quad \tau=t\bmod L,
\end{equation}
i.e., samples with the same $\tau$ are drawn from the same distribution regardless of the absolute time index.

\paragraph{Cycle-dependent location-scale shift.}
A common and practically relevant family of shifts is the channel-wise location-scale form:
\begin{equation}
X_t = \mu(\tau) + \sigma(\tau)\odot Z_t,\qquad \tau=t\bmod L,
\label{eq:14}
\end{equation}
where $Z_t$ is an approximately cycle-invariant process, and $\mu(\tau)$ and $\sigma(\tau)$ encode structured \emph{mean offsets} and \emph{variance fluctuations} within each cycle.
Based on Equation (\ref{eq:14}), the optimal predictor is inherently $\tau$-conditional:
\begin{equation}
f_{\theta}^\star(X_t,\tau)=\mathbb{E}[Y_t\,|\,X_t,\tau],
\end{equation}
and ignoring $\tau$ forces a single hypothesis to fit a mixture of heterogeneous distributions $\{\mathcal{D}_\tau\}_{\tau=0}^{L-1}$, which typically increases approximation error and harms out-of-distribution performance across time.

\paragraph{PAMod as cycle-aware mean and variance adaptation.}
According to Equation \eqref{eq:3} and \eqref{eq:12}, the modulated prediction can be expressed as:
\begin{equation}
    \begin{aligned}
        \hat{Y}_{\text{PAMod}}&=\sigma_X\cdot{f_{\theta}}(X_{\text{norm}, \tau})+\mu_X\\
        &=\sigma_X\cdot f_{\theta}(X_{\text{norm}}+X_{\text{mod}}) + \mu_X.
        \label{eq:16}
    \end{aligned}
\end{equation}
Assuming the neural network is approximately linear and the modulation is small relative to the normalized input, PAMod's prediction can be approximated as:
\begin{equation}
    \begin{aligned}
        \hat{Y}_{\text{PAMod}}&=\sigma_X\cdot f_{\theta}(X_{\text{norm}}) \odot (1 + \alpha(\tau))\\
        &+ (\sigma_X\cdot \phi(\tau)+\mu_X),
    \end{aligned}
\end{equation}
where $\alpha(\tau)$ and $\phi(\tau)$ represent the learned amplitude $U_2E_A(t)$ and phase $U_3E_P(t)$, respectively. This formulation shows that our PAMod can adaptively adjust both the mean and variance of the base prediction $f_{\theta}(X_{\text{norm}})$ based on the cyclical position $\tau$ to model distribution shifts.
Moreover, we provide the performance guarantee of PAMod under periodic location-scale shifts in Appendix \ref{app:a}.

\begin{table*}[!hb]
\centering
\vspace{-0.3cm}
\caption{Multivariate forecasting performance. The lookback length is set to $L=96$ and all the results are averaged from all predictions $H\in\{12, 24, 48, 96\}$ for PEMS and $H\in\{96, 192, 336, 720\}$ for other benchmarks. See Table \ref{tab:8} in the Appendix for the full results.}
\label{tab:1}
\scriptsize
\setlength{\tabcolsep}{1.2mm}{
\begin{tabular}{c|cc|cc|cc|cc|cc|cc|cc|cc|cc|cc}
\specialrule{0.15em}{0pt}{2pt}
\multirow{2}{*}{Model} & \multicolumn{2}{c|}{\textbf{PAMod}}  & \multicolumn{2}{c|}{TQNet}   & \multicolumn{2}{c|}{TimeEmb} & \multicolumn{2}{c|}{FilterTS} & \multicolumn{2}{c|}{Amplifier} & \multicolumn{2}{c|}{TimeXer} & \multicolumn{2}{c|}{CycleNet} & \multicolumn{2}{c|}{TimeMixer} & \multicolumn{2}{c|}{iTransformer} & \multicolumn{2}{c}{PatchTST} \\ 
\cmidrule(lr){2-3} \cmidrule(lr){4-5} \cmidrule(lr){6-7} \cmidrule(lr){8-9} \cmidrule(lr){10-11} \cmidrule(lr){12-13} \cmidrule(lr){14-15} \cmidrule(lr){16-17} \cmidrule(lr){18-19} \cmidrule(lr){20-21}
                        & \multicolumn{2}{c|}{\textbf{(Ours)}} & \multicolumn{2}{c|}{\citeyearpar{DBLP:conf/icml/LinCWQL25}}    & \multicolumn{2}{c|}{\citeyearpar{DBLP:journals/corr/abs-2510-00461}}       & \multicolumn{2}{c|}{\citeyearpar{DBLP:conf/aaai/WangLDW25}}     & \multicolumn{2}{c|}{\citeyearpar{DBLP:conf/aaai/Fei000N25}}      & \multicolumn{2}{c|}{\citeyearpar{DBLP:conf/nips/WangWDQZLQWL24}}       & \multicolumn{2}{c|}{\citeyearpar{DBLP:conf/nips/Lin0HWMZ24}}     & \multicolumn{2}{c|}{\citeyearpar{DBLP:conf/iclr/WangWSHLMZ024}}       & \multicolumn{2}{c|}{\citeyearpar{DBLP:conf/iclr/LiuHZWWML24}}         & \multicolumn{2}{c}{\citeyearpar{DBLP:conf/iclr/NieNSK23}}     \\  
                       \specialrule{0.10em}{1pt}{1pt}
Metric                 & MSE               & MAE              & MSE            & MAE         & MSE            & MAE         & MSE           & MAE           & MSE            & MAE           & MSE           & MAE          & MSE           & MAE           & MSE            & MAE           & MSE                  & MAE        & MSE           & MAE          \\ 
\specialrule{0.10em}{1pt}{1pt}
ETTm1                  & \textbf{0.363}    & \textbf{0.378}   & 0.377          & 0.393       & {\underline{0.368}}    & {\underline{0.385}} & 0.385         & 0.396         & 0.382          & 0.395         & 0.382         & 0.397        & 0.379         & 0.396         & 0.381          & 0.395         & 0.407                & 0.410      & 0.387         & 0.400        \\ 
\specialrule{0.10em}{1pt}{1pt}
ETTm2                  & \textbf{0.263}    & \textbf{0.307}   & 0.277          & 0.323       & {\underline{0.265}}    & {\underline{0.308}} & 0.277         & 0.322         & 0.280          & 0.326         & 0.274         & 0.322        & 0.266         & 0.314         & 0.275          & 0.323         & 0.288                & 0.332      & 0.281         & 0.326        \\ 
\specialrule{0.10em}{1pt}{1pt}
ETTh1                  & \textbf{0.413}    & \textbf{0.420}   & 0.441          & 0.434       & {\underline{0.425}}    & {\underline{0.425}} & 0.434         & 0.430         & 0.430          & 0.428         & 0.437         & 0.437        & 0.457         & 0.441         & 0.447          & 0.440         & 0.454                & 0.448      & 0.469         & 0.455        \\ 
\specialrule{0.10em}{1pt}{1pt}
ETTh2                  & \textbf{0.360}    & \textbf{0.387}   & 0.378          & 0.402       & {\underline{0.362}}    & {\underline{0.390}} & 0.375         & 0.398         & 0.381          & 0.405         & 0.368         & 0.396        & 0.388         & 0.409         & 0.364          & 0.395         & 0.383                & 0.407      & 0.387         & 0.407        \\ 
\specialrule{0.10em}{1pt}{1pt}
ECL                    & {\underline{0.165}}       & \textbf{0.254}   & \textbf{0.164} & {\underline{0.259}} & 0.168          & 0.261       & 0.180         & 0.272         & 0.172          & 0.266         & 0.171         & 0.270        & 0.168         & 0.259         & 0.182          & 0.272         & 0.178                & 0.270      & 0.205         & 0.290        \\ 
\specialrule{0.10em}{1pt}{1pt}
Traffic                & {\underline{0.434}}       & \textbf{0.272}   & 0.445          & {\underline{0.276}} & 0.453          & 0.295       & 0.470         & 0.315         & 0.483          & 0.317         & 0.466         & 0.287        & 0.472         & 0.314         & 0.484          & 0.297         & \textbf{0.428}       & 0.282      & 0.481         & 0.300        \\ 
\specialrule{0.10em}{1pt}{1pt}
Weather                & {\underline{0.239}}       & \textbf{0.261}   & 0.242          & 0.269       & \textbf{0.237} & {\underline{0.262}} & 0.245         & 0.274         & 0.253          & 0.275         & 0.241         & 0.271        & 0.243         & 0.271         & 0.240          & 0.271         & 0.258                & 0.278      & 0.259         & 0.273        \\ 
\specialrule{0.10em}{1pt}{1pt}
Solar                  & {\underline{0.206}}       & \textbf{0.226}   & \textbf{0.198} & {\underline{0.256}} & 0.248          & 0.270       & 0.215         & 0.277         & 0.241          & 0.270         & 0.237         & 0.302        & 0.210         & 0.261         & 0.216          & 0.280         & 0.233                & 0.262      & 0.270         & 0.307        \\ 
\specialrule{0.10em}{1pt}{1pt}
PEMS03                 & \textbf{0.091}    & \textbf{0.189}   & {\underline{0.097}}    & {\underline{0.203}} & 0.104          & 0.207       & 0.134         & 0.246         & 0.131          & 0.239         & 0.112         & 0.214        & 0.118         & 0.226         & 0.167          & 0.267         & 0.113                & 0.222      & 0.180         & 0.291        \\ 
\specialrule{0.10em}{1pt}{1pt}
PEMS04                 & \textbf{0.086}    & \textbf{0.185}   & {\underline{0.091}}    & {\underline{0.197}} & 0.096          & 0.200       & 0.125         & 0.241         & 0.135          & 0.249         & 0.105         & 0.209        & 0.119         & 0.232         & 0.185          & 0.287         & 0.111                & 0.221      & 0.195         & 0.307        \\ 
\specialrule{0.10em}{1pt}{1pt}
PEMS07                 & {\underline{0.080}}       & \textbf{0.168}   & \textbf{0.075} & {\underline{0.171}} & 0.097          & 0.188       & 0.120         & 0.220         & 0.122          & 0.226         & 0.085         & 0.182        & 0.113         & 0.214         & 0.181          & 0.271         & 0.101                & 0.204      & 0.211         & 0.303        \\ 
\specialrule{0.10em}{1pt}{1pt}
PEMS08                 & \textbf{0.119}    & \textbf{0.190}   & 0.142          & 0.229       & {\underline{0.132}}    & {\underline{0.221}} & 0.180         & 0.266         & 0.183          & 0.271         & 0.175         & 0.250        & 0.150         & 0.246         & 0.226          & 0.299         & 0.150                & 0.226      & 0.280         & 0.321        \\ 
\specialrule{0.10em}{1pt}{1pt}
$1^{st}$ Count         & \textbf{7}        & \textbf{12}      & {\underline{3}}        & 0           & 1              & 0           & 0             & 0             & 0              & 0             & 0             & 0            & 0             & 0             & 0              & 0             & 1                    & 0          & 0             & 0            \\ 
\specialrule{0.15em}{2pt}{0pt}
\end{tabular}}
\end{table*}

\section{Experiments}
In this section, we conduct comprehensive experiments with real-world time series benchmarks to sufficiently assess the performance of our proposed PAMod, including comparison with SOTA baselines
on performance and efficiency (Section \ref{subsec:4.2}), ablation studies (Section \ref{subsec:4.3}), compatibility analysis (Section \ref{subsec:4.4}), and interpretable cases (Section \ref{subsec:4.5}).

\subsection{Experimental Settings}
\textbf{Datasets}.
Following established evaluation protocols \cite{DBLP:conf/iclr/LiuHZWWML24, DBLP:conf/icml/LinCWQL25} in time series forecasting literature, we perform experiments on twelve real-world benchmark datasets. 
These include four ETT datasets (ETTh1, ETTh2, ETTm1, ETTm2), Electricity (ECL), Traffic, Solar-Energy, Weather, and four PEMS datasets (PEMS03, PEMS04, PEMS07, PEMS08) \cite{DBLP:conf/aaai/ZhouZPZLXZ21, DBLP:conf/nips/WuXWL21, DBLP:conf/nips/LiuZCXLM022}.
In line with common practice, each dataset is partitioned into training, validation, and test sets. 
For the ETT series, we adopt a 6:2:2 split ratio, while a 7:1:2 ratio is used for all other datasets.
The prediction horizons $H$ are set to $\{12, 24, 48, 96\}$ for the PEMS datasets and $\{96, 192, 336, 720\}$ for the others. 

\textbf{Baselines}.
We compare PAMod against nine state-of-the-art models in recent years, including TQNet \cite{DBLP:conf/icml/LinCWQL25}, TimeEmb \cite{DBLP:journals/corr/abs-2510-00461}, FilterTS \cite{DBLP:conf/aaai/WangLDW25}, Amplifier \cite{DBLP:conf/aaai/Fei000N25}, TimeXer \cite{DBLP:conf/nips/WangWDQZLQWL24}, CycleNet \cite{DBLP:conf/nips/Lin0HWMZ24}, TimeMixer \cite{DBLP:conf/iclr/WangWSHLMZ024}, iTransformer \cite{DBLP:conf/iclr/LiuHZWWML24}, and PatchTST \cite{DBLP:conf/iclr/NieNSK23}.
In line with standard evaluation protocols, a uniform look-back length of 96 is adopted for all models to ensure a fair comparison.

\textbf{Implementation Details}.
All experiments are implemented in PyTorch \cite{DBLP:conf/nips/PaszkeGMLBCKLGA19} and trained on a single NVIDIA GeForce RTX 4090 GPU (24GB), using the Adam optimizer \cite{DBLP:journals/corr/KingmaB14}. 
To ensure stable training, we follow prior work \cite{DBLP:conf/iclr/0049PS0YY0L0T25, DBLP:conf/icml/LiuWHL0BX25} in employing a hybrid loss function that combines Mean Absolute Error (MAE) in both the time and frequency domains. The key hyperparameter $\tau$, which defines the periodicity length, is set for each dataset following the established guidelines from CycleNet and TQNet. 
Model performance is evaluated using the standard metrics of Mean Squared Error (MSE) and Mean Absolute Error (MAE).

\subsection{Main Results}
\label{subsec:4.2}

\textbf{Forecasting Performance}.
Table \ref{tab:1} shows the comparison results of PAMod with 9 recent methods across 12 real-world time series benchmarks.
Lower MSE and MAE values indicate better forecasting performance.
Overall, PAMod consistently achieves Top 2 performance across all datasets, with the smallest error metrics in 19 out of 24 evaluation cases, demonstrating new state-of-the-art accuracy.
Notably, an increasing number of variables in a dataset tends to amplify cross-series non-stationarity, which in turn increases the difficulty of forecasting.
Nevertheless, PAMod still delivers competitive and even better performance.
For instance, while TQNet employs cross-attention mechanisms to capture temporal dependencies and achieves an average MSE of $0.445$ on the Traffic dataset, PAMod, which leverages lightweight phase-amplitude modulation to explicitly model non-stationary dynamics, attains a lower average MSE of $0.434$ on the same dataset.
This empirical result confirms that explicitly modeling non-stationarity is a more effective strategy for capturing complex temporal dependencies.

\begin{figure*}[ht]
\centering
\includegraphics[width=0.95\linewidth]{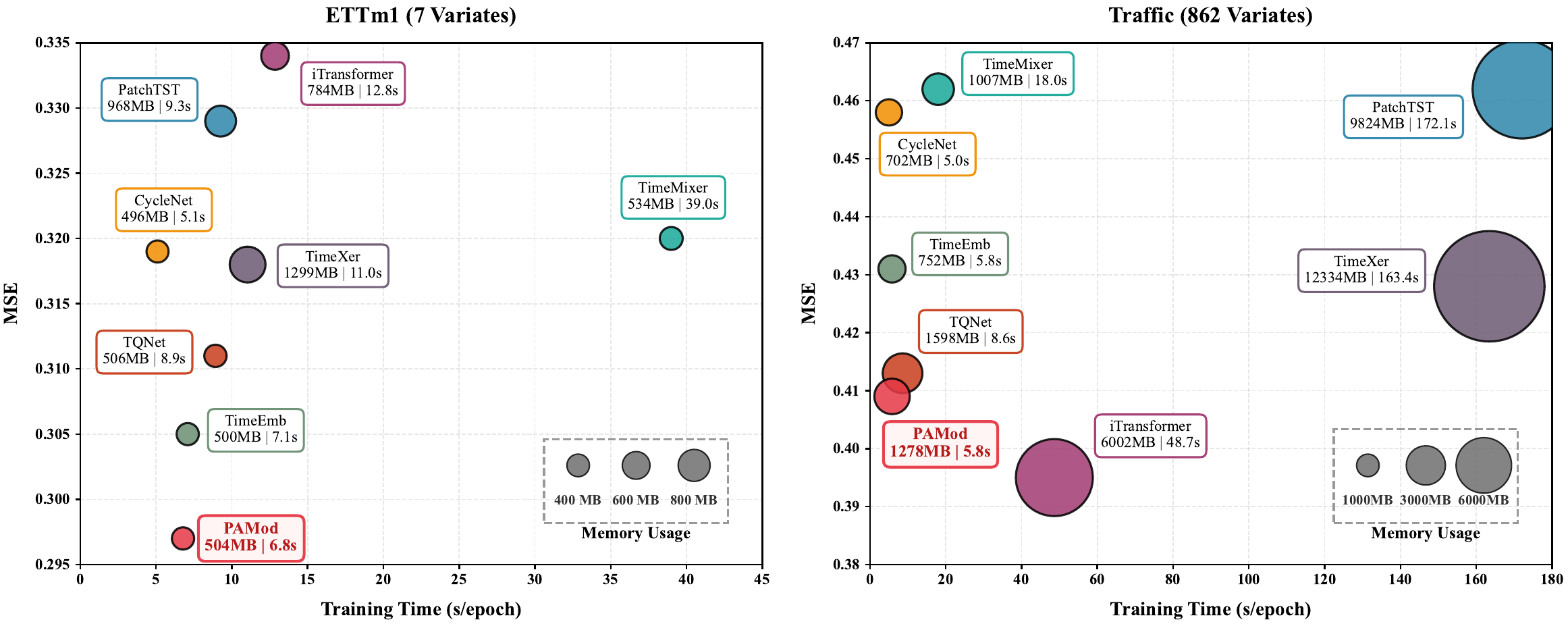}
\vspace{-0.15cm}
\caption{Computational efficiency under the input-96-predict-96 setting. Comparison of forecasting performance (MSE), memory usage, and training time on the ETTm1 and Traffic datasets. We set the batch sizes to 32 and 16 for the ETTm1 and Traffic datasets, respectively.}
\label{fig:3}
\vspace{-0.35cm}
\end{figure*}

\textbf{Model Efficiency}.
To evaluate the efficiency of PAMod, we record the training time and memory footprint of PAMod with seven baseline models under the same batch size for a fair comparison.
As shown in Figure \ref{fig:3}, PAMod demonstrates a superior balance of prediction accuracy, training speed, and memory efficiency across both small- and large-scale forecasting scenarios. 
On the ETTm1 dataset (7 Variates), it achieves the lowest MSE (0.297) while maintaining training time second only to CycleNet and moderate memory consumption. When evaluated in the Traffic dataset (862 variates), PAMod not only retains competitive accuracy but also exhibits near-optimal training speed and drastically reduced memory usage—surpassing strong baselines by more than an order of magnitude in both time and memory efficiency. 
These results underscore PAMod’s scalable and resource-efficient design, making it well-suited for real-world deployment where computational constraints are as critical as predictive performance.

\begin{table*}[hb]
\centering
\vspace{-0.9cm}
\caption{Ablation studies. The lookback length and prediction horizon are set to $96$ for all the datasets.}
\vspace{-0.2cm}
\label{tab:2}
\scriptsize
\setlength{\tabcolsep}{1.2mm}{
\begin{tabular}{c|cc|cccccccccccc|cccc}
\specialrule{0.15em}{0pt}{1pt}
\multirow{3}{*}{Datasets} & \multicolumn{2}{c|}{\multirow{2}{*}{\textbf{PAMod}}} & \multicolumn{12}{c|}{Modulation Variants}                                                                                                                                                                                     & \multicolumn{4}{c}{Training loss}                            \\ 
\cmidrule(lr){4-15}\cmidrule(lr){16-19}
                          & \multicolumn{2}{c|}{}                                 & \multicolumn{2}{c|}{w/o $\phi(\tau)$}     & \multicolumn{2}{c|}{w/o $\alpha(\tau)$} & \multicolumn{2}{c|}{r/w LRC}       & \multicolumn{2}{l|}{r/w $\sin(\tau)$} & \multicolumn{2}{c|}{s/w $\phi(\tau)~\&~\alpha(\tau)$}          & \multicolumn{2}{c|}{w/o modulation} & \multicolumn{2}{c|}{MAE}           & \multicolumn{2}{c}{MSE} \\ 
                          \cmidrule(lr){2-19}
                          & MSE                       & MAE                       & MSE   & \multicolumn{1}{c|}{MAE}   & MSE   & \multicolumn{1}{c|}{MAE}   & MSE   & \multicolumn{1}{c|}{MAE}   & MSE    & \multicolumn{1}{c|}{MAE}   & MSE   & \multicolumn{1}{c|}{MAE}   & MSE              & MAE              & MSE   & \multicolumn{1}{c|}{MAE}   & MSE        & MAE        \\ 
                          \specialrule{0.10em}{0.5pt}{0.5pt}
ETTm1                     & \textbf{0.297}            & \textbf{0.337}            & 0.302 & \multicolumn{1}{c|}{0.341} & 0.304 & \multicolumn{1}{c|}{0.342} & 0.305 & \multicolumn{1}{c|}{0.344} & 0.308  & \multicolumn{1}{c|}{0.347} & 0.301 & \multicolumn{1}{c|}{0.340} & 0.324            & 0.355            & 0.294 & \multicolumn{1}{c|}{0.331} & 0.296      & 0.340      \\ 
\specialrule{0.10em}{0.5pt}{0.5pt}
ETTm2                     & \textbf{0.162}            & \textbf{0.240}            & 0.168 & \multicolumn{1}{c|}{0.245} & 0.168 & \multicolumn{1}{c|}{0.246} & 0.163 & \multicolumn{1}{c|}{0.241} & 0.168  & \multicolumn{1}{c|}{0.245} & 0.163 & \multicolumn{1}{c|}{0.241} & 0.175            & 0.253            & 0.164 & \multicolumn{1}{c|}{0.240} & 0.164      & 0.242      \\ 
\specialrule{0.10em}{0.5pt}{0.5pt}
ETTh1                     & \textbf{0.357}            & \textbf{0.382}            & 0.361 & \multicolumn{1}{c|}{0.384} & 0.365 & \multicolumn{1}{c|}{0.385} & 0.359 & \multicolumn{1}{c|}{0.383} & 0.363  & \multicolumn{1}{c|}{0.386} & 0.359 & \multicolumn{1}{c|}{0.383} & 0.371            & 0.387            & 0.362 & \multicolumn{1}{c|}{0.379} & 0.360      & 0.386      \\ 
\specialrule{0.10em}{0.5pt}{0.5pt}
ETTh2                     & \textbf{0.279}            & \textbf{0.328}            & 0.284 & \multicolumn{1}{c|}{0.330} & 0.287 & \multicolumn{1}{c|}{0.335} & 0.280 & \multicolumn{1}{c|}{0.327} & 0.284  & \multicolumn{1}{c|}{0.330} & 0.281 & \multicolumn{1}{c|}{0.328} & 0.298            & 0.341            & 0.284 & \multicolumn{1}{c|}{0.330} & 0.283      & 0.330      \\ 
\specialrule{0.10em}{0.5pt}{0.5pt}
ECL                       & \textbf{0.136}            & \textbf{0.226}            & 0.139 & \multicolumn{1}{c|}{0.228} & 0.156 & \multicolumn{1}{c|}{0.245} & 0.141 & \multicolumn{1}{c|}{0.229} & 0.145  & \multicolumn{1}{c|}{0.232} & 0.142 & \multicolumn{1}{c|}{0.230} & 0.181            & 0.256            & 0.138 & \multicolumn{1}{c|}{0.225} & 0.137      & 0.229      \\ 
\specialrule{0.10em}{0.5pt}{0.5pt}
Traffic                   & \textbf{0.409}            & \textbf{0.256}            & 0.413 & \multicolumn{1}{c|}{0.259} & 0.418 & \multicolumn{1}{c|}{0.265} & 0.417 & \multicolumn{1}{c|}{0.256} & 0.425  & \multicolumn{1}{c|}{0.262} & 0.420 & \multicolumn{1}{c|}{0.257} & 0.497            & 0.310            & 0.412 & \multicolumn{1}{c|}{0.250} & 0.410      & 0.258      \\ 
\specialrule{0.10em}{0.5pt}{0.5pt}
Weather                   & \textbf{0.153}            & \textbf{0.191}            & 0.156 & \multicolumn{1}{c|}{0.194} & 0.168 & \multicolumn{1}{c|}{0.203} & 0.153 & \multicolumn{1}{c|}{0.191} & 0.158  & \multicolumn{1}{c|}{0.194} & 0.156 & \multicolumn{1}{c|}{0.193} & 0.194            & 0.224            & 0.154 & \multicolumn{1}{c|}{0.191} & 0.155      & 0.193      \\ 
\specialrule{0.10em}{0.5pt}{0.5pt}
Solar                     & \textbf{0.184}            & \textbf{0.209}            & 0.188 & \multicolumn{1}{c|}{0.212} & 0.188 & \multicolumn{1}{c|}{0.205} & 0.197 & \multicolumn{1}{c|}{0.216} & 0.199  & \multicolumn{1}{c|}{0.220} & 0.185 & \multicolumn{1}{c|}{0.210} & 0.239            & 0.261            & 0.191 & \multicolumn{1}{c|}{0.207} & 0.184      & 0.210      \\ 
\specialrule{0.10em}{0.5pt}{0.5pt}
PEMS03                    & \textbf{0.133}            & \textbf{0.228}            & 0.139 & \multicolumn{1}{c|}{0.234} & 0.182 & \multicolumn{1}{c|}{0.280} & 0.133 & \multicolumn{1}{c|}{0.228} & 0.142  & \multicolumn{1}{c|}{0.238} & 0.135 & \multicolumn{1}{c|}{0.230} & 0.247            & 0.346            & 0.132 & \multicolumn{1}{c|}{0.226} & 0.132      & 0.230      \\ 
\specialrule{0.10em}{0.5pt}{0.5pt}
PEMS04                    & \textbf{0.111}            & \textbf{0.212}            & 0.114 & \multicolumn{1}{c|}{0.216} & 0.188 & \multicolumn{1}{c|}{0.287} & 0.115 & \multicolumn{1}{c|}{0.215} & 0.123  & \multicolumn{1}{c|}{0.219} & 0.113 & \multicolumn{1}{c|}{0.214} & 0.278            & 0.365            & 0.112 & \multicolumn{1}{c|}{0.212} & 0.113      & 0.213      \\ 
\specialrule{0.10em}{0.5pt}{0.5pt}
PEMS07                    & \textbf{0.116}            & \textbf{0.199}            & 0.121 & \multicolumn{1}{c|}{0.203} & 0.172 & \multicolumn{1}{c|}{0.268} & 0.122 & \multicolumn{1}{c|}{0.207} & 0.132  & \multicolumn{1}{c|}{0.211} & 0.128 & \multicolumn{1}{c|}{0.212} & 0.317            & 0.373            & 0.114 & \multicolumn{1}{c|}{0.198} & 0.110      & 0.201      \\ 
\specialrule{0.10em}{0.5pt}{0.5pt}
PEMS08                    & \textbf{0.181}            & \textbf{0.223}            & 0.186 & \multicolumn{1}{c|}{0.228} & 0.210 & \multicolumn{1}{c|}{0.283} & 0.189 & \multicolumn{1}{c|}{0.226} & 0.205  & \multicolumn{1}{c|}{0.247} & 0.195 & \multicolumn{1}{c|}{0.262} & 0.348            & 0.378            & 0.179 & \multicolumn{1}{c|}{0.224} & 0.175      & 0.226      \\ 
\specialrule{0.10em}{0.5pt}{0.5pt}
Avg                       & \textbf{0.210}            & \textbf{0.253}            & 0.214 & \multicolumn{1}{c|}{0.256} & 0.234 & \multicolumn{1}{c|}{0.279} & 0.215 & \multicolumn{1}{c|}{0.255} & 0.221  & \multicolumn{1}{c|}{0.261} & 0.215 & \multicolumn{1}{c|}{0.258} & 0.289            & 0.321            & 0.212 & \multicolumn{1}{c|}{0.251} & 0.210      & 0.255      \\ 
\specialrule{0.15em}{1pt}{0pt}
\end{tabular}}
\vspace{-0.3cm}
\end{table*}

\subsection{Ablation Studies}
\label{subsec:4.3}
To investigate the role of each component in PAMod, we perform comprehensive ablation studies on the modulation mechanism, training loss, and cycle length.

\textbf{Components}.~~
As shown in Table~\ref{tab:2}, $\phi(\tau)$ and $\alpha(\tau)$ denote cycle-based phase bias and amplitude modulation, respectively. 
`r/w LRC' introduces the learnable recurrent cycle as CycleNet's to replace $\phi(\tau)$ and $\alpha(\tau)$.
Meanwhile, we also design the $\tau$-based Sinusoidal function to replace $\phi(\tau)$ and $\alpha(\tau)$ as `r/w $\sin(\tau)$', and swap $\phi(\tau)$ and $\alpha(\tau)$ as `s/w $\phi(\tau)~\&~\alpha(\tau)$'.
The results of Modulation Variants in Table~\ref{tab:2} are inferior to those of PAMod, indicating the effectiveness of our phase-amplitude modulation mechanism.
Moreover,  the pure MLP model (i.e., w/o modulation) obtains the worst results, confirming the importance of explicitly modeling cyclical shifts.
Additionally, adopting either MAE or MSE loss has a negligible impact on the forecasting performance, which further highlights that the accuracy improvements of PAMod originate primarily from its explicit modeling of cycle-dependent phase and amplitude dynamics.

\newpage

\begin{table*}[htbp]
\centering
\caption{Performance comparison of integrating our PAMod mechanism with different backbone models on the ETTh2 and ETTm1 datasets. The best results are highlighted in \textbf{bold}, and ``Imp." denotes the performance improvement achieved by incorporating PAMod.}
\vspace{-0.1cm}
\label{tab:3}
\scriptsize
\setlength{\tabcolsep}{1.5mm}{
\begin{tabular}{c|cccccccc|cccccccc}
\specialrule{0.15em}{0pt}{1pt}
Dataset      & \multicolumn{8}{c|}{ETTh2}                                                                                                                                                                           & \multicolumn{8}{c}{ETTm1}                                                                                                                                                                            \\ 
\specialrule{0.10em}{0.5pt}{0.5pt}
Horizon      & \multicolumn{2}{c|}{96}                              & \multicolumn{2}{c|}{192}                             & \multicolumn{2}{c|}{336}                             & \multicolumn{2}{c|}{720}        & \multicolumn{2}{c|}{96}                              & \multicolumn{2}{c|}{192}                             & \multicolumn{2}{c|}{336}                             & \multicolumn{2}{c}{720}         \\ 
\specialrule{0.10em}{0.5pt}{0.5pt}
Metric       & MSE            & \multicolumn{1}{c|}{MAE}            & MSE            & \multicolumn{1}{c|}{MAE}            & MSE            & \multicolumn{1}{c|}{MAE}            & MSE            & MAE            & MSE            & \multicolumn{1}{c|}{MAE}            & MSE            & \multicolumn{1}{c|}{MAE}            & MSE            & \multicolumn{1}{c|}{MAE}            & MSE            & MAE            \\ 
\specialrule{0.10em}{0.5pt}{0.5pt}
DLinear \citeyearpar{DBLP:conf/aaai/ZengCZ023}     & 0.333          & \multicolumn{1}{c|}{0.387}          & 0.477          & \multicolumn{1}{c|}{0.476}          & 0.594          & \multicolumn{1}{c|}{0.541}          & 0.831          & 0.657          & 0.345          & \multicolumn{1}{c|}{0.372}          & 0.380          & \multicolumn{1}{c|}{0.389}          & 0.413          & \multicolumn{1}{c|}{0.413}          & 0.474          & 0.453          \\
+PAMod       & \textbf{0.312} & \multicolumn{1}{c|}{\textbf{0.342}} & \textbf{0.426} & \multicolumn{1}{c|}{\textbf{0.433}} & \textbf{0.445} & \multicolumn{1}{c|}{\textbf{0.472}} & \textbf{0.484} & \textbf{0.489} & \textbf{0.317} & \multicolumn{1}{c|}{\textbf{0.348}} & \textbf{0.355} & \multicolumn{1}{c|}{\textbf{0.372}} & \textbf{0.395} & \multicolumn{1}{c|}{\textbf{0.392}} & \textbf{0.448} & \textbf{0.432} \\ 
\specialrule{0.10em}{0.5pt}{0.5pt}
Imp.         & 6.3\%          & \multicolumn{1}{c|}{11.6\%}         & 10.7\%         & \multicolumn{1}{c|}{9.0\%}          & 25.1\%         & \multicolumn{1}{c|}{12.8\%}         & 41.8\%         & 25.6\%         & 8.1\%          & \multicolumn{1}{c|}{6.5\%}          & 6.6\%          & \multicolumn{1}{c|}{4.4\%}          & 4.4\%          & \multicolumn{1}{c|}{5.1\%}          & 5.5\%          & 4.6\%          \\ 
\specialrule{0.10em}{0.5pt}{0.5pt}
PatchTST \citeyearpar{DBLP:conf/iclr/NieNSK23}    & 0.302          & \multicolumn{1}{c|}{0.348}          & 0.388          & \multicolumn{1}{c|}{0.400}          & 0.426          & \multicolumn{1}{c|}{0.433}          & 0.431          & 0.466          & 0.329          & \multicolumn{1}{c|}{0.367}          & 0.367          & \multicolumn{1}{c|}{0.385}          & 0.399          & \multicolumn{1}{c|}{0.410}          & 0.545          & 0.439          \\
+PAMod       & \textbf{0.293} & \multicolumn{1}{c|}{\textbf{0.334}} & \textbf{0.365} & \multicolumn{1}{c|}{\textbf{0.388}} & \textbf{0.401} & \multicolumn{1}{c|}{\textbf{0.418}} & \textbf{0.419} & \textbf{0.438} & \textbf{0.312} & \multicolumn{1}{c|}{\textbf{0.345}} & \textbf{0.352} & \multicolumn{1}{c|}{\textbf{0.370}} & \textbf{0.380} & \multicolumn{1}{c|}{\textbf{0.394}} & \textbf{0.478} & \textbf{0.428} \\ 
\specialrule{0.10em}{0.5pt}{0.5pt}
Imp.         & 3.0\%          & \multicolumn{1}{c|}{4.0\%}          & 5.9\%          & \multicolumn{1}{c|}{3.0\%}          & 5.9\%          & \multicolumn{1}{c|}{3.5\%}          & 2.8\%          & 6.0\%          & 5.2\%          & \multicolumn{1}{c|}{6.0\%}          & 4.1\%          & \multicolumn{1}{c|}{3.9\%}          & 4.8\%          & \multicolumn{1}{c|}{3.9\%}          & 12.3\%         & 2.5\%          \\
\specialrule{0.10em}{0.5pt}{0.5pt}
iTransformer \citeyearpar{DBLP:conf/iclr/LiuHZWWML24} & 0.297          & \multicolumn{1}{c|}{0.349}          & 0.380          & \multicolumn{1}{c|}{0.400}          & 0.428          & \multicolumn{1}{c|}{0.432}          & 0.427          & 0.445          & 0.334          & \multicolumn{1}{c|}{0.368}          & 0.377          & \multicolumn{1}{c|}{0.391}          & 0.426          & \multicolumn{1}{c|}{0.420}          & 0.491          & 0.459          \\
+PAMod       & \textbf{0.284} & \multicolumn{1}{c|}{\textbf{0.332}} & \textbf{0.363} & \multicolumn{1}{c|}{\textbf{0.385}} & \textbf{0.406} & \multicolumn{1}{c|}{\textbf{0.414}} & \textbf{0.410} & \textbf{0.431} & \textbf{0.312} & \multicolumn{1}{c|}{\textbf{0.349}} & \textbf{0.354} & \multicolumn{1}{c|}{\textbf{0.376}} & \textbf{0.393} & \multicolumn{1}{c|}{\textbf{0.399}} & \textbf{0.446} & \textbf{0.434} \\ 
\specialrule{0.10em}{0.5pt}{0.5pt}
Imp.         & 4.4\%          & \multicolumn{1}{c|}{4.9\%}          & 4.5\%          & \multicolumn{1}{c|}{3.8\%}          & 5.1\%          & \multicolumn{1}{c|}{4.2\%}          & 4.0\%          & 3.1\%          & 6.6\%          & \multicolumn{1}{c|}{5.2\%}          & 6.1\%          & \multicolumn{1}{c|}{3.8\%}          & 7.7\%          & \multicolumn{1}{c|}{5.0\%}          & 9.2\%          & 5.4\%          \\ 
\specialrule{0.10em}{0.5pt}{0.5pt}
TQNet \citeyearpar{DBLP:conf/icml/LinCWQL25}       & 0.295          & \multicolumn{1}{c|}{0.343}          & 0.367          & \multicolumn{1}{c|}{0.393}          & 0.417          & \multicolumn{1}{c|}{0.427}          & 0.433          & 0.446          & 0.311          & \multicolumn{1}{c|}{0.353}          & 0.356          & \multicolumn{1}{c|}{0.378}          & 0.390          & \multicolumn{1}{c|}{0.401}          & 0.452          & 0.440          \\
+PAMod       & \textbf{0.280} & \multicolumn{1}{c|}{\textbf{0.329}} & \textbf{0.358} & \multicolumn{1}{c|}{\textbf{0.378}} & \textbf{0.399} & \multicolumn{1}{c|}{\textbf{0.413}} & \textbf{0.412} & \textbf{0.433} & \textbf{0.299} & \multicolumn{1}{c|}{\textbf{0.337}} & \textbf{0.347} & \multicolumn{1}{c|}{\textbf{0.369}} & \textbf{0.375} & \multicolumn{1}{c|}{\textbf{0.388}} & \textbf{0.435} & \textbf{0.425} \\ 
\specialrule{0.10em}{0.5pt}{0.5pt}
Imp.         & 5.1\%          & \multicolumn{1}{c|}{4.1\%}          & 2.5\%          & \multicolumn{1}{c|}{3.8\%}          & 4.3\%          & \multicolumn{1}{c|}{3.3\%}          & 4.8\%          & 2.9\%          & 3.9\%          & \multicolumn{1}{c|}{4.5\%}          & 2.5\%          & \multicolumn{1}{c|}{2.4\%}          & 3.8\%          & \multicolumn{1}{c|}{3.2\%}          & 3.8\%          & 3.4\%          \\ 
\specialrule{0.15em}{1pt}{0pt}
\end{tabular}}
\vspace{-0.2cm}
\end{table*}

\begin{figure}[ht]
\centering
\includegraphics[width=1.0\linewidth]{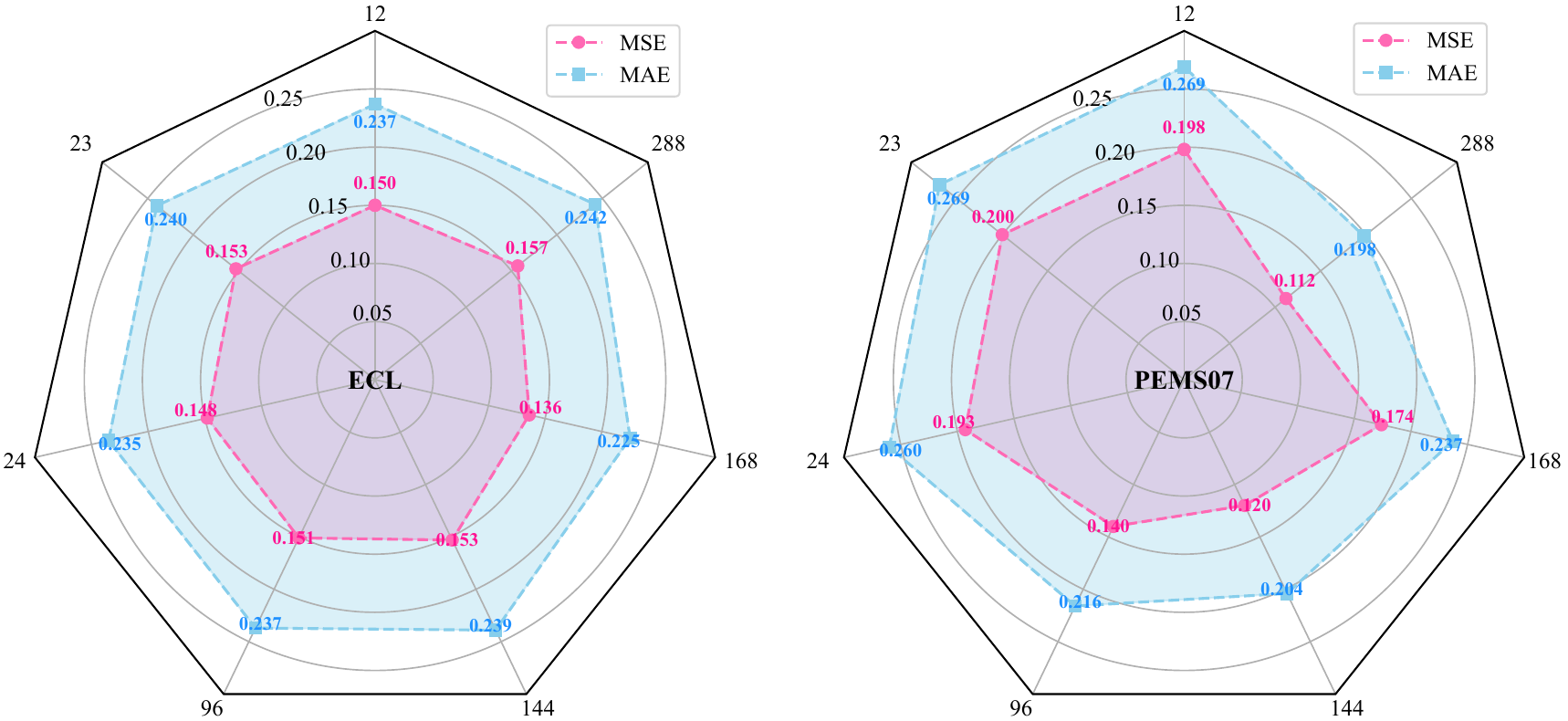}
\vspace{-0.35cm}
\caption{Performance of PAMod on the ECL and PEMS07 with varying cycle length $\tau\in\{12, 23, 24, 96, 144, 168, 288\}$.}
\label{fig:4}
\vspace{-0.15cm}
\end{figure}

\textbf{Cycle Lengths}.~
The learnable phase and amplitude embeddings $\phi(\tau)$ and $\alpha(\tau)$ in our PAMod mechanism are explicitly conditioned on the cycle length $\tau$. 
To investigate its critical role, we conduct a systematic study of $\tau$ across the ECL and PEMS07 datasets. 
As shown in Figure \ref{fig:4}, performance peaks when $\tau$ aligns with the inherent periodicity defined by the dataset's recording frequency.
For instance, in the hourly-sampled ECL dataset, the daily cycle ($\tau=24$) substantially outperforms non-periodic lengths, and the weekly cycle ($\tau=24\times 7=168$) achieves optimal results. 
These findings confirm that distribution shifts are structurally aligned with the maximum periodicity \cite{DBLP:conf/nips/Lin0HWMZ24, DBLP:conf/icml/LinCWQL25}. 
Therefore, aligning $\tau$ accordingly is essential: it allows our model to explicitly capture these structured cyclical variations, which is the key to improving forecasting performance under non-stationarity.

\subsection{Compatibility Analysis}
\label{subsec:4.4}
To evaluate the compatibility and versatility of our phase-amplitude modulation (PAMod) mechanism, we seamlessly integrate it as a plug-and-play module into a diverse set of backbones, including both MLP-based and Transformer-based architectures. 
Crucially, PAMod is designed to be architecture-agnostic, capable of functioning as a complementary learnable embedding module or as a direct replacement for existing normalization techniques.

\begin{table}[ht]
\centering
\caption{MSE error comparison with normalization methods.}
\vspace{-0.1cm}
\label{tab:4}
\scriptsize
\setlength{\tabcolsep}{0.9mm}{
\begin{tabular}{cc|cccc|cccc}
\specialrule{0.15em}{0pt}{1pt}
\multicolumn{2}{c|}{Backbone}                                            & \multicolumn{4}{c|}{Autoformer \citeyearpar{DBLP:conf/nips/WuXWL21}}        & \multicolumn{4}{c}{SparseTSF \citeyearpar{DBLP:conf/icml/Lin0WCY24}}          \\ 
\specialrule{0.10em}{0.5pt}{0.5pt}
\multicolumn{2}{c|}{Norm Type}                                           & RevIN & SAN   & DDN   & \textbf{PAMod} & RevIN & SAN   & DDN   & PAMod          \\ 
\specialrule{0.10em}{0.5pt}{0.5pt}
\multicolumn{1}{c|}{\multirow{5}{*}{Weather}} & 96                       & 0.212 & 0.194 & 0.190 & \textbf{0.188} & 0.186 & 0.182 & 0.178 & \textbf{0.164} \\
\multicolumn{1}{c|}{}                         & 192                      & 0.264 & 0.258 & 0.231 & \textbf{0.227} & 0.231 & 0.229 & 0.224 & \textbf{0.216} \\
\multicolumn{1}{c|}{}                         & 336                      & 0.309 & 0.329 & 0.289 & \textbf{0.281} & 0.285 & 0.283 & 0.279 & \textbf{0.273} \\
\multicolumn{1}{c|}{}                         & 720                      & 0.377 & 0.440 & 0.369 & \textbf{0.354} & 0.360 & 0.362 & 0.358 & \textbf{0.351} \\ 
\cmidrule(lr){2-10}
\multicolumn{1}{c|}{}                         & \multicolumn{1}{l|}{Avg} & 0.291 & 0.305 & 0.270 & \textbf{0.263} & 0.266 & 0.264 & 0.260 & \textbf{0.251} \\ 
\specialrule{0.10em}{0.5pt}{0.5pt}
\multicolumn{1}{c|}{\multirow{5}{*}{ECL}}     & 96                       & 0.179 & 0.172 & 0.150 & \textbf{0.148} & 0.202 & 0.189 & 0.185 & \textbf{0.137} \\
\multicolumn{1}{c|}{}                         & 192                      & 0.216 & 0.195 & 0.173 & \textbf{0.169} & 0.199 & 0.192 & 0.188 & \textbf{0.154} \\
\multicolumn{1}{c|}{}                         & 336                      & 0.233 & 0.211 & 0.185 & \textbf{0.180} & 0.212 & 0.201 & 0.194 & \textbf{0.170} \\
\multicolumn{1}{c|}{}                         & 720                      & 0.246 & 0.236 & 0.201 & \textbf{0.199} & 0.253 & 0.248 & 0.235 & \textbf{0.208} \\ 
\cmidrule(lr){2-10}
\multicolumn{1}{c|}{}                         & \multicolumn{1}{l|}{Avg} & 0.219 & 0.204 & 0.177 & \textbf{0.174} & 0.217 & 0.208 & 0.201 & \textbf{0.167} \\ 
\specialrule{0.15em}{1pt}{0pt}
\end{tabular}}
\vspace{-0.3cm}
\end{table}

As shown in Tables \ref{tab:3}, PAMod consistently improves performance when integrated with distinct backbones on ETTh2 and ETTm1.
To be concrete, PAMod reduces MSE by 2.5\%–41.8\% and MAE by 2.4\%–25.6\% compared to respective backbones alone.
In Table \ref{tab:4}, PAMod also outperforms specialized normalization methods, i.e., RevIN \cite{DBLP:conf/iclr/KimKTPCC22}, SAN \cite{DBLP:conf/nips/LiuCLHLXC23}, and DDN \cite{DBLP:conf/nips/0001WLLYXZ24}, on the Weather and ECL datasets.
These results validate that PAMod is not architecture-specific. Its ability to flexibly enhance diverse backbones stems from its design: the phase-amplitude modulation acts on general feature representations, making it a lightweight, plug-in module for improving non-stationary time series forecasting.

\begin{figure}[hb]
\centering
\includegraphics[width=1.0\linewidth]{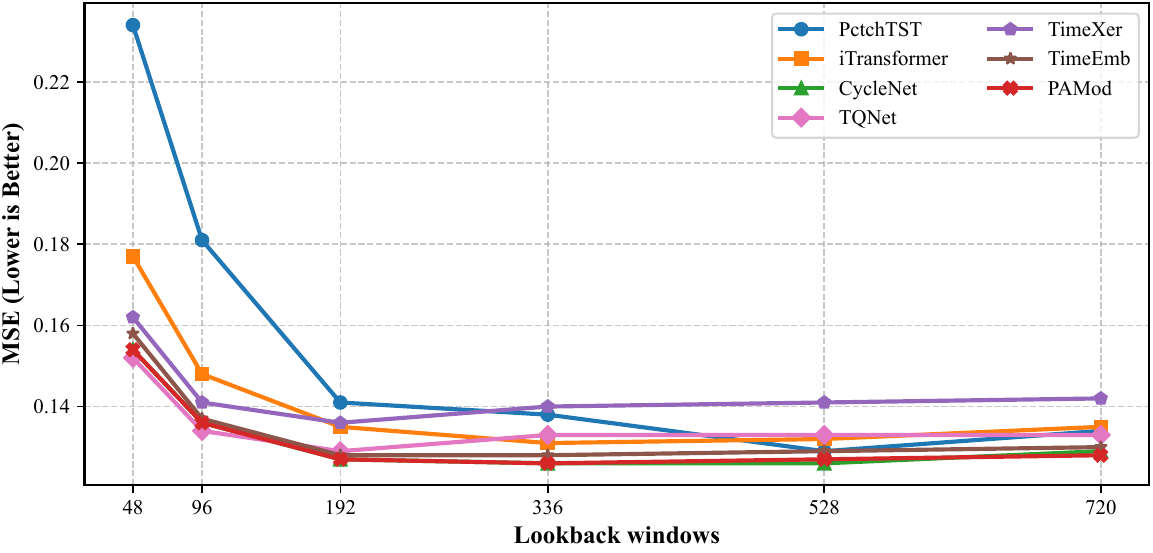}
\vspace{-0.45cm}
\caption{Performance of different methods under varying lookback windows $T\in\{48, 96, 192, 336, 528, 720\}$ on the ECL dataset. The prediction horizon is fixed as 96.}
\label{fig:5}
\vspace{-0.1cm}
\end{figure}

The length of the lookback window determines the richness of historical information available to the model. To understand this relationship, we investigate the model's sensitivity to different lookback window lengths.
As shown in Figure \ref{fig:5}, PAMod's performance scales positively with the lookback window length while maintaining a consistent lead over baselines.
This scaling effect demonstrates the efficacy of using learnable embeddings to model distribution shifts. 
By dynamically capturing the evolving phase (mean) and amplitude (variance) of cyclical patterns, these embeddings provide a stable representational basis that makes long historical sequences interpretable, enabling the model to leverage extended contexts and capture genuine long-term dependencies.
See Appendix \ref{app:c} for more results.

\subsection{Interpretable Cases}
\label{subsec:4.5}

To gain an intuitive understanding of PAMod's capability to model cyclical shifts, we visualize both the learned embedding weights and the corresponding prediction outcomes.

\begin{figure}[htbp]
\centering
\includegraphics[width=1.0\linewidth]{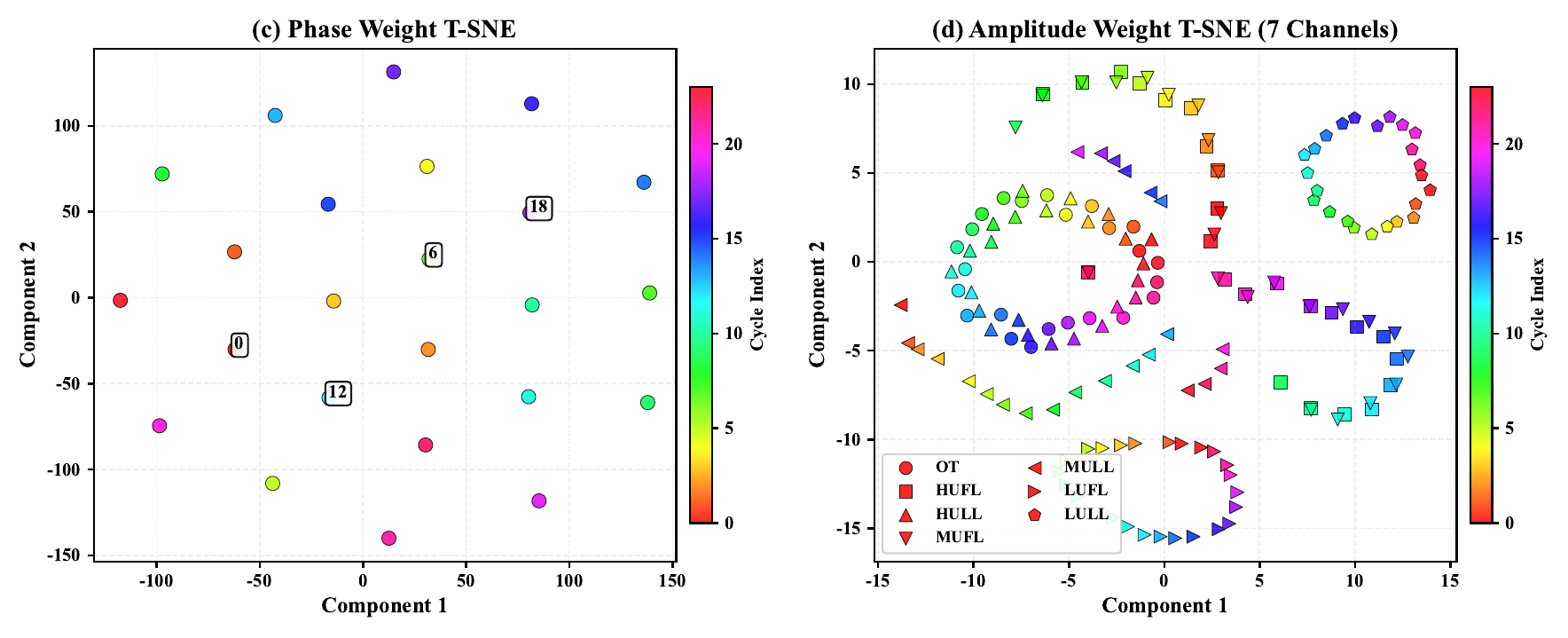}
\vspace{-0.35cm}
\caption{t-SNE \cite{maaten2008visualizing} visualization of learned phase and amplitude weights on the ETTh1 dataset.}
\label{fig:6}
\end{figure}

\textbf{Phase-Amplitude Weights}.~ As shown in Figure \ref{fig:6}, the t-SNE projections of the phase and amplitude weights exhibit a clear structural disentanglement, directly reflecting their designed roles. 
The phase weight cluster is strictly aligned according to the cycle index, demonstrating its primary function in capturing temporal mean shifts. 
Conversely, the amplitude weight forms distinct clusters based on channel identity, indicating that it learns to modulate channel-specific variance. 
Such clean separation confirms that our phase-amplitude modulation successfully decouples the temporal (phase) and channel-wise (amplitude) factors of distribution shifts, providing an explicit and interpretable mechanism to model non-stationarity in multivariate time series.

\begin{figure}[htbp]
\centering
\includegraphics[width=1.0\linewidth]{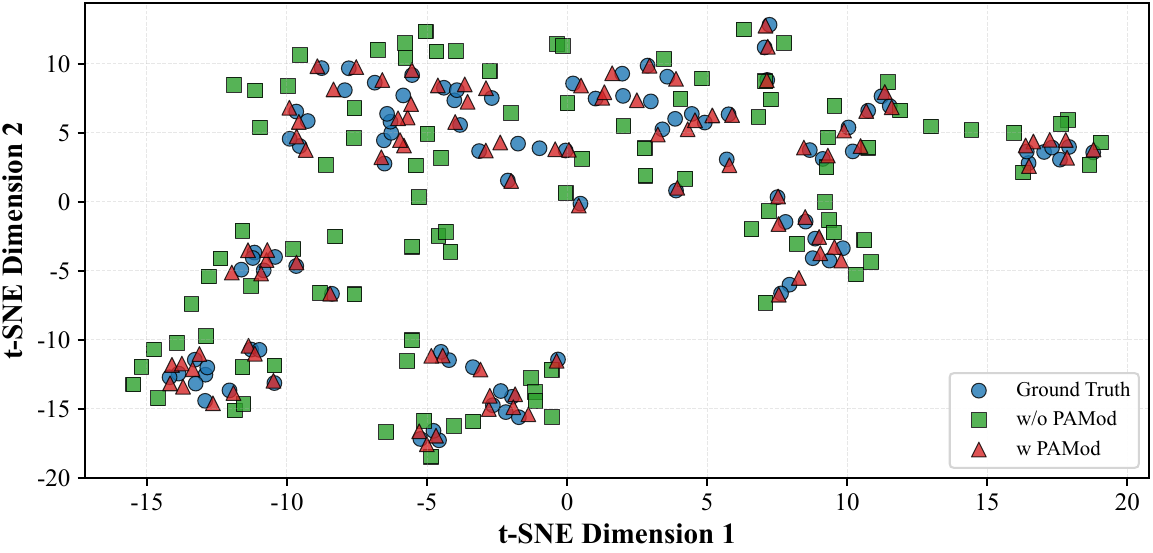}
\vspace{-0.35cm}
\caption{Comparison of prediction and true distributions using t-SNE on the ETTh1 dataset.}
\label{fig:7}
\vspace{-0.3cm}
\end{figure}

\textbf{Distribution Alignment}.~ 
Figure \ref{fig:7} provides geometric evidence that PAMod aligns predictions with the ground truth under distribution shift. 
The ground truth forms a compact, structured cluster, representing the intrinsic data manifold.  
Predictions with PAMod closely overlap this cluster, maintaining its structure and continuity, indicating high distributional consistency. 
In contrast, the predictions without PAMod are widely scattered and deviate markedly from the true manifold.
This visual evidence validates that our phase-amplitude mechanism explicitly models temporal distribution drift, therefore enabling predictions to reside on the true data manifold.
Furthermore, Figure \ref{fig:8} shows that PAMod achieves a closer fit to the ground truth than TimeEmb, accurately replicating both its long-term trend and local variations.
This confirms its enhanced capacity for modeling complex temporal dynamics and achieving fine-grained temporal alignment.

\begin{figure}[htbp]
\centering
\includegraphics[width=1.0\linewidth]{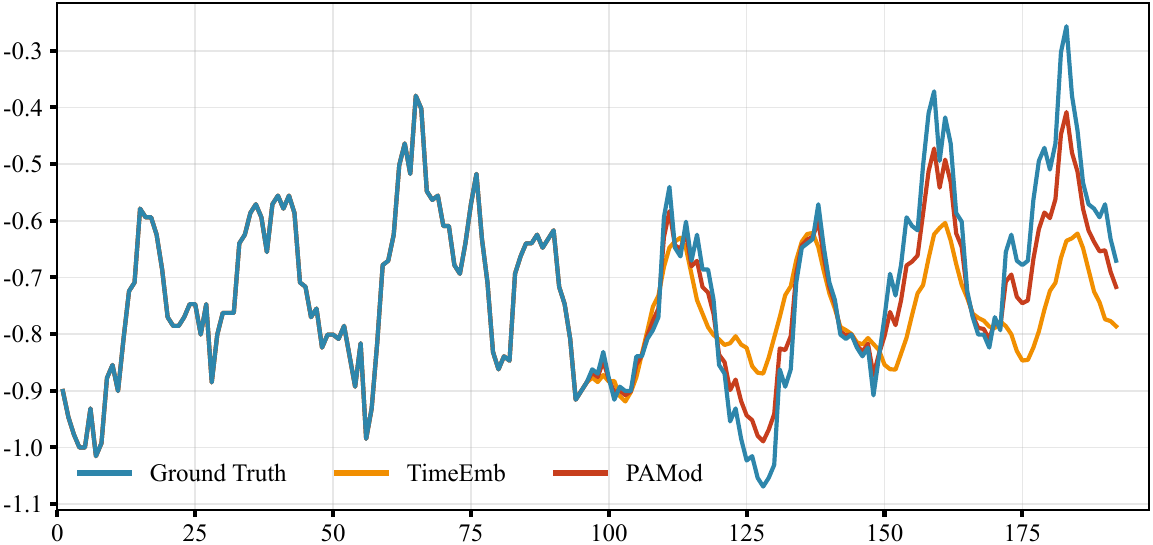}
\vspace{-0.35cm}
\caption{Visualization of forecasting performance between PAMod and TimeEmb on the ETTh1 dataset.}
\label{fig:8}
\vspace{-0.3cm}
\end{figure}

\section{Conclusion}
Real-world time series exhibit inherent non-stationarity, which often manifests as structured, cyclical distribution shifts over time. 
To address this challenge, we propose \textbf{PAMod}, a \textbf{P}hase‑\textbf{A}mplitude \textbf{Mod}ulation framework that explicitly models cyclical shifts in mean (phase) and variance (amplitude) through learnable periodic embeddings. 
Operating in normalized feature space, PAMod performs dynamic distribution adaptation equivalent to learnable denormalization, unifying shift modeling with representation learning in a lightweight, plug‑and‑play module. Extensive experiments validate that PAMod achieves state‑of‑the‑art forecasting accuracy across diverse real‑world benchmarks while maintaining high parameter efficiency, demonstrating that explicitly modeling cyclical non‑stationarity is both effective and essential for robust time series forecasting.
Future work may explore extending PAMod to model multi-scale periodicities simultaneously and adapt the modulation mechanism to other time series tasks.

\clearpage

\section*{Impact Statement}

This paper presents work whose goal is to advance the field of Machine Learning. 
There are many potential societal consequences of our work, none of which we feel must be specifically highlighted here.

\bibliography{example_paper}
\bibliographystyle{icml2026}

\newpage
\appendix
\onecolumn
\section{Performance Guarantee of PAMod}
\label{app:a}

In this section, we provide a formal theoretical justification for the Phase-Amplitude Modulation (PAMod) framework. 
We demonstrate that by conditioning the model on cyclical indices, PAMod minimizes the generalization error bound under periodic non-stationarity.

\subsection{Discrepancy under Domain Shift}

To analyze the shift between different cyclical distributions, we adopt the Discrepancy Distance ($d_{\text{disc}}$), which generalizes the $\mathcal{H}\Delta\mathcal{H}$-divergence from binary classification to continuous output spaces and general loss functions \cite{DBLP:conf/colt/MansourMR09}.

\begin{definition}
\textbf{(Discrepancy Distance)}. Let $\mathcal{H}$ be a hypothesis class of functions mapping $\mathcal{X} \to \mathcal{Y}$. The discrepancy between two distributions $\mathcal{D}_i$ and $\mathcal{D}_j$ is defined as:
\begin{equation}
    d_{\text{disc}}(\mathcal{D}_i, \mathcal{D}_j) = \sup_{h, h' \in \mathcal{H}} \left| \mathbb{E}_{x \sim \mathcal{D}_i} \mathcal{L}(h(x), h'(x)) - \mathbb{E}_{x \sim \mathcal{D}_j} \mathcal{L}(h(x), h'(x)) \right|,
\end{equation}
where $\mathcal{L}$ is the loss function (e.g., Mean Squared Error). This metric is particularly suited for our framework as it captures the variations in both mean (phase) and variance (amplitude) within the hypothesis space.
\end{definition}

\subsection{Generalization Bound under Periodic Shift}

We assume the time series follows the periodic location-scale shift defined in Equation \eqref{eq:14}: $X_t = \mu(\tau) + \sigma(\tau) \odot Z_t$, where $\tau = t \bmod L$ and $Z_t$ is a cycle-invariant latent process. Let $\psi(X, \tau)$ denote the PAMod operator, and $\hat{\mathcal{D}}_\tau$ be the distribution of the modulated features.

\begin{theorem}
\textbf{(Generalization Bound \cite{DBLP:conf/colt/MansourMR09, DBLP:journals/ml/Ben-DavidBCKPV10})}. For any cyclical position $\tau \in \{0, \dots, L-1\}$, the expected risk $\mathcal{R}_\tau(h)$ of the predictor is bounded by:
\begin{equation}
    \mathcal{R}_\tau(h) \leq \hat{\mathcal{R}}_S(h) + d_{\text{disc}}(\hat{\mathcal{D}}_\tau, \mathcal{D}_{ref}) + \lambda,
\end{equation}
where $\hat{\mathcal{R}}_S(h)$ is the empirical risk on the training set, $d_{\text{disc}}(\hat{\mathcal{D}}_\tau, \mathcal{D}_{ref})$ is the discrepancy between the modulated distribution $\hat{\mathcal{D}}_\tau$ and a stationary reference distribution $\mathcal{D}_{ref}$ (the distribution of $Z$), and $\lambda$ is the error of the ideal joint hypothesis.
\end{theorem}

\textbf{Proof Sketch}: By the triangle inequality of the discrepancy distance, the error on a target distribution $\mathcal{D}_\tau$ can be decomposed into the source error and the distance between the two distributions. PAMod implements a transformation $\psi(X, \tau) = X \odot E_A(\tau) + E_P(\tau)$. Under the assumption in Eq. (14), if the model learns $E_A(\tau) \approx \sigma(\tau)^{-1}$ and $E_P(\tau) \approx -\mu(\tau)\sigma(\tau)^{-1}$, the transformed distribution $\hat{\mathcal{D}}_\tau$ converges to the stationary distribution $\mathcal{D}_{ref}$. Consequently:
\begin{equation}
    \inf_{E_A, E_P} d_{\text{disc}}(\psi(\mathcal{D}_\tau), \mathcal{D}_{ref}) \to 0.
\end{equation}
Thus, PAMod minimizes the upper bound of the generalization error by explicitly neutralizing the cycle-dependent distribution discrepancy, which remains a large constant in models without cycle-awareness.

\subsection{Lipschitz Continuity and Stability}
To ensure that the modulation does not amplify input perturbations, we provide a stability guarantee.

\begin{lemma}
\textbf{(Lipschitz Continuity of PAMod)}.
The modulation operator $\psi(X, \tau)$ is Lipschitz continuous with respect to $X$. Specifically, if the linear projections $U_i$ have bounded weights $\|U\| \leq W_{\max}$, then:
\begin{equation}
    \|\psi(X_1, \tau) - \psi(X_2, \tau)\| \leq K \|X_1 - X_2\|,
\end{equation}
where $K$ depends on $W_{\max}$ and the norm of the embeddings.
\end{lemma}

\textbf{Discussion}: This continuity ensures that the learnable cyclical embeddings do not introduce high-frequency noise or instability into the latent space. By combining the Discrepancy Minimization (Theorem A.2.) and Lipschitz Stability (Lemma A.3.), PAMod provides a robust performance guarantee for long-term time series forecasting even under severe periodic distribution drifts.

\section{More Details of PAMod}
\label{app:b}

\subsection{Framework Implementation}
The pseudocode of PAMod is presented in Algorithm \ref{alg:1}, where the key design lies in our phase-amplitude modulation. 

\vspace{-0.2cm}
\begin{algorithm}[htbp]
\caption{Pseudocode of PAMod.}
\label{alg:1}
\begin{algorithmic}[1]
\REQUIRE 
    Lookback length $\mathbf{X}\in\mathbb{R}^{T\times C}$, cycle length $L$
\ENSURE
    The prediction horizon $\hat{\mathbf{X}}\in\mathbb{R}^{H\times C}$\\
\STATE Initialize learnable embedding matrix $\Omega_p\in\mathbb{R}^{c\times T}$ and $\Omega_a\in\mathbb{R}^{L\times CT}$ with the Xavier normal distribution
\IF{Instance Normalization is True}
\STATE $\mu, \sigma \leftarrow Mean(X), STD(X)$
\STATE $X_{\text{norm}}\leftarrow \frac{X-\mu}{\sqrt{\sigma^2 + \epsilon}}$
\ENDIF
\STATE \CommentLeft{Cycle Index}
\STATE $T_{end}\leftarrow{\text{time stamp of the last observed step}}$
\STATE $\tau=T_{end}\mod L$
\STATE \CommentLeft{Phase Embedding}
\FOR{$i \in \{1, \cdots, C\}$}
     \STATE $E_P^i(t)\leftarrow{\text{Lookup($W_m, t$)}}$
     \COMMENT{$E_P^i(t)\in\mathbb{R}^{1\times T}$}
\ENDFOR
\STATE $E_P(t)\leftarrow \text{contact}(E_p^1, \cdots, E_p^C)$
\COMMENT{$E_P(t)\in\mathbb{R}^{C\times T}$}
\STATE \CommentLeft{Amplitude Embedding}
\STATE $E_A(t)\leftarrow \text{Reshape(Lookup}(W_v, t)$
\COMMENT{$E_A(t)\in \mathbb{R}^{C\times T}$}
\STATE \CommentLeft{Phase-Amplitude Modulation}
\STATE $X_1=\text{ReLU}(U_1X_{\text{norm}})$
\COMMENT{$X_1\in\mathbb{R}^{C\times T}$}
\STATE $X_{\text{mod}}=\text{Dropout}(U_4(X_1\odot U_2E_A(t)+U_3E_P(t)))$
\COMMENT{$X_{\text{mod}}\in\mathbb{R}^{C\times T}$}
\STATE \CommentLeft{Forecasting Head}
\STATE $\hat{Y}=\text{MLP}(X_{\text{mod}}+X_{\text{norm}})$
\STATE $\hat{\mathbf{Y}} = \hat{\mathbf{Y}}.transpose(-1, -2)$
\COMMENT{$\hat{\mathbf{Y}}\in \mathbb{R}^{H\times C}$}
\IF{Instance Normalization is True}
\STATE $\hat{Y}\leftarrow \hat{Y}\times \sqrt{\sigma^2 + \epsilon} + \mu$
\ENDIF
\end{algorithmic}
\end{algorithm} 
\vspace{-0.2cm}

\subsection{Benchmarks details}
We evaluate the performance of PAMod compared with various baselines on $12$ well-established benchmarks~\footnote{All the datasets are publicly available at \url{https://github.com/thuml/iTransformer}}, which are detailed in Table~\ref{tab:5}.
\vspace{-0.2cm}
\begin{itemize}
    \item \textbf{ETT (Electricity Transformer Temperature)} comprises two hourly-level datasets (i.e., ETTh1 and ETTh2) and two 15-minute-level datasets (i.e., ETTm1 and ETTm2). Each dataset contains seven oil and load features of electricity transformers from July 2016 to July 2018.
    \item \textbf{Electricity} encompasses the hourly electricity consumption data of 321 customers from 2012 to 2014.
    \item \textbf{Traffic} describes the road occupancy rates from the California Department of Transportation. It contains the hourly data recorded by the sensors of San Francisco freeways from 2015 to 2016.
    \item \textbf{Solar} records the solar power production of 137 PV plants in 2006, which is sampled every 10 minutes.
    \item \textbf{Weather} includes 21 indicators of weather, such as air temperature, and humidity. Its data is recorded every 10 min for 2020 in Germany.
    \item \textbf{PEMS} contains public traffic network data in California collected by 5-minute windows.
\end{itemize}
We follow the same data processing and train-validation-test set split protocol of TimesNet~\citep{DBLP:conf/iclr/WuHLZ0L23}, where each part is strictly divided according to chronological order without data leakage issues.

\begin{table}[htbp]
\centering
\caption{Detailed descriptions of datasets. Channel denotes the number of variates in each dataset. Prediction length points out four prediction settings. The dataset size is split into (Train, Validation, Test). Frequency denotes the sampling interval of time points.}
\label{tab:5}
\scriptsize
\begin{tabular}{c|c|c|c|c|c}
\specialrule{0.15em}{0pt}{3pt}
Benchmarks   & Channels & Prediction Length  & Dataset Size  & Frequency & Cycle    \\ 
\specialrule{0.15em}{2pt}{0pt}
 ETTm1   & 7 & \multirow{9}{*}{$\{$96, 192, 336, 720$\}$} & (34465, 11521, 11521) & 15min     & 96    \\ 
\cline{1-2} \cline{4-6}
ETTm2        & 7        &             & (34465, 11521, 11521) & 15min     & 96    \\ \cline{1-2} \cline{4-6}  
 ETTh1        & 7        &      & (8545, 2881, 2881)    & Hourly    & 24    \\ \cline{1-2} \cline{4-6}  
 ETTh2        & 7        &       & (8545, 2881, 2881)    & Hourly    & 24    \\ \cline{1-2} \cline{4-6} 
 ECL          & 321      &       & (18317, 2633, 5261)   & Hourly    & 168    \\ \cline{1-2} \cline{4-6}  
 Traffic      & 862      &    & (12185, 1757, 3509)   & Hourly    & 168 \\ \cline{1-2} \cline{4-6}  
  Weather      & 21       &     & (36792, 5271, 10540)  & 10min     & 144        \\ \cline{1-2} \cline{4-6}  
Solar-energy & 137      &       & (36601, 5161, 10417)  & 10min     & 144   \\ 
\hline
 PEMS03       & 358      & \multirow{4}{*}{$\{$12, 24, 48, 96$\}$}    & (15617, 5135, 5135)   & 5min      & 288 \\ \cline{1-2} \cline{4-6} 
  PEMS04       & 307      &     & (10172, 3375, 3375)   & 5min      & 288 \\ \cline{1-2} \cline{4-6} 
 PEMS07       & 883      &     & (16911, 5622, 5622)   & 5min      & 288 \\ \cline{1-2} \cline{4-6}  
 PEMS08       & 170      &   & (10690, 3548, 3548)   & 5min      & 288 \\ 
\specialrule{0.15em}{0pt}{0pt}
\end{tabular}
\end{table}

\subsection{Experimental Details}
PAMod is trained for 30 epochs with early stopping based on a patience of 5 on the validation set.
During the training process, the learning rate is set to $[1\times 10^{-3}, 5\times 10^{-3}]$, the embedding dimension is set to $512$, and the dropout rate is set to 0.5 by default.
To improve training stability, we adopt a hybrid MAE loss \cite{DBLP:conf/iclr/0049PS0YY0L0T25} that operates in both the time and frequency domains:
\begin{equation}
    \begin{aligned}
        \mathcal{L}_t=\frac{1}{H}\sum_{i=1}^{H}\lvert Y&-\hat{Y}\rvert, ~
        \mathcal{L}_f=\frac{1}{H}\sum_{i=1}^{H}\lvert \mathcal{F}(Y)- \mathcal{F}(\hat{Y})\rvert, \\
        \mathcal{L} &= (1-\alpha)\times \mathcal{L}_t + \alpha\times \mathcal{L}_f,
    \end{aligned}
\end{equation}
where $\mathcal{F}$ denotes the Fast Fourier Transform, and the hyperparameter $\alpha$ is set to $[0.05, 0.35]$ for different datasets.
We use Mean Squared Error (MSE) and Mean Absolute Error (MAE) as evaluation metrics. 
Given the ground truth values $Y\in \mathbb{R}^{H\times C}$ and the predicted values $\hat{Y}\in \mathbb{R}^{H\times C}$, these metrics are defined as:
\begin{equation}
        \text{MSE}=\frac{1}{H}\sum_{i=1}^{H}( Y-\hat{Y})^2, ~~~\text{MAE}=\frac{1}{H}\sum_{i=1}^{H}\lvert Y- \hat{Y}\rvert.
\end{equation}

\section{More Results of PAMod}
\label{app:c}

\subsection{Error Bars}
We obtain the standard deviation of PAMod performance by training the model with 5 different random seeds over 12 datasets.
As seen in Table~\ref{tab:6}, the error bars of all the results are tiny, indicating our PAMod is robust and reliable.

\begin{table}[htbp]
\centering
\caption{Robustness of PAMod performance obtained from 5 random seeds on 12 benchmarks.}
\label{tab:6}
\scriptsize
\setlength{\tabcolsep}{1.6mm}{
\begin{tabular}{c|cc|cc|cc|cc}
\specialrule{0.15em}{0pt}{2pt}
Dataset & \multicolumn{2}{c|}{ETTm1}          & \multicolumn{2}{c|}{ETTm2}          & \multicolumn{2}{c|}{ETTh1}          & \multicolumn{2}{c}{ETTh2}             \\ 
\cmidrule(lr){1-1} \cmidrule(lr){2-3} \cmidrule(lr){4-5} \cmidrule(lr){6-7} \cmidrule(lr){8-9}
Metrics & MSE              & MAE              & MSE              & MAE              & MSE              & MAE              & MSE              & MAE              \\ 
\specialrule{0.10em}{2pt}{2pt}
96      & $0.0.297\pm 0.001$ & $0.337\pm 0.001$ & $0.162\pm 0.000$ & $0.240\pm 0.001$ & $0.357\pm 0.001$ & $0.382\pm 0.000$ & $0.279\pm 0.000$ & $0.328\pm 0.000$ \\
192     & $0.346\pm 0.001$ & $0.366\pm 0.001$ & $0.226\pm 0.001$ & $0.283\pm 0.000$ & $0.403\pm 0.002$ & $0.413\pm 0.001$ & $0.349\pm 0.001$ & $0.374\pm 0.001$ \\
336     & $0.375\pm 0.001$ & $0.387\pm 0.002$ & $0.283\pm 0.001$ & $0.321\pm 0.001$ & $0.443\pm 0.004$ & $0.434\pm 0.002$ & $0.403\pm 0.002$ & $0.414\pm 0.001$ \\
720     & $0.432\pm 0.002$ & $0.422\pm 0.001$ & $0.382\pm 0.002$ & $0.382\pm 0.001$ & $0.447\pm 0.003$ & $0.452\pm 0.001$ & $0.409\pm 0.002$ & $0.432\pm 0.002$ \\ 
\specialrule{0.10em}{2pt}{2pt}
Dataset & \multicolumn{2}{c|}{ECL}        & \multicolumn{2}{c|}{Traffic}       & \multicolumn{2}{c|}{Weather}   & \multicolumn{2}{c}{Solar}         \\ 
\cmidrule(lr){1-1} \cmidrule(lr){2-3} \cmidrule(lr){4-5} \cmidrule(lr){6-7} \cmidrule(lr){8-9}
Metrics & MSE              & MAE              & MSE              & MAE              & MSE              & MAE              & MSE              & MAE              \\ 
\specialrule{0.10em}{2pt}{2pt}
96      & $0.136\pm 0.001$ & $0.226\pm 0.000$ & $0.409\pm 0.001$ & $0.256\pm 0.001$ & $0.153\pm 0.002$ & $0.191\pm 0.002$ & $0.184\pm 0.000$ & $0.209\pm 0.000$ \\
192     & $0.151\pm 0.002$ & $0.240\pm 0.001$ & $0.424\pm 0.002$ & $0.268\pm 0.001$ & $0.201\pm 0.001$ & $0.237\pm 0.001$ & $0.205\pm 0.001$ & $0.226\pm 0.001$ \\
336     & $0.168\pm 0.003$ & $0.258\pm 0.003$ & $0.436\pm 0.002$ & $0.274\pm 0.001$ & $0.259\pm 0.000$ & $0.281\pm 0.001$ & $0.215\pm 0.002$ & $0.232\pm 0.001$ \\
720     & $0.205\pm 0.001$ & $0.290\pm 0.002$ & $0.465\pm 0.002$ & $0.291\pm 0.003$ & $0.341\pm 0.001$ & $0.336\pm 0.002$ & $0.216\pm 0.002$ & $0.238\pm 0.002$ \\ 
\specialrule{0.10em}{2pt}{2pt}
Dataset & \multicolumn{2}{c|}{PEMS03}         & \multicolumn{2}{c|}{PEMS04}         & \multicolumn{2}{c|}{PEMS07}         & \multicolumn{2}{c}{PEMS08}          \\ 
\cmidrule(lr){1-1} \cmidrule(lr){2-3} \cmidrule(lr){4-5} \cmidrule(lr){6-7} \cmidrule(lr){8-9}
Metrics & MSE              & MAE              & MSE              & MAE              & MSE              & MAE              & MSE              & MAE              \\ 
\specialrule{0.10em}{2pt}{2pt}
12      & $0.059\pm 0.000$ & $0.157\pm 0.000$ & $0.066\pm 0.001$ & $0.162\pm 0.001$ & $0.053\pm 0.000$ & $0.141\pm 0.000$ & $0.074\pm 0.000$ & $0.161\pm 0.001$ \\
24      & $0.072\pm 0.000$ & $0.173\pm 0.000$ & $0.075\pm 0.001$ & $0.174\pm 0.000$ & $0.065\pm 0.000$ & $0.156\pm 0.000$ & $0.094\pm 0.000$ & $0.177\pm 0.000$ \\
48      & $0.098\pm 0.001$ & $0.199\pm 0.001$ & $0.091\pm 0.002$ & $0.192\pm 0.001$ & $0.086\pm 0.001$ & $0.176\pm 0.001$ & $0.126\pm 0.001$ & $0.199\pm 0.001$ \\
96      & $0.133\pm 0.002$ & $0.228\pm 0.001$ & $0.111\pm 0.001$ & $0.212\pm 0.001$ & $0.116\pm 0.001$ & $0.199\pm 0.001$ & $0.181\pm 0.002$ & $0.223\pm 0.001$ \\ 
\specialrule{0.15em}{2pt}{0pt}
\end{tabular}}
\vspace{-0.3cm}
\end{table}

\subsection{Extra Quantitative Results}

\begin{table}[htbp]
\centering
\caption{Performance comparison of integrating our PAMod mechanism with different backbone models on the ETTh1 and ETTm2 datasets. The best results are highlighted in \textbf{bold}, and ``Imp." denotes the performance improvement achieved by incorporating PAMod.}
\vspace{-0.1cm}
\label{tab:7}
\scriptsize
\setlength{\tabcolsep}{1.5mm}{
\begin{tabular}{c|cccccccc|cccccccc}
\specialrule{0.15em}{0pt}{1pt}
Dataset      & \multicolumn{8}{c|}{ETTh1}                                                                                                                                                                           & \multicolumn{8}{c}{ETTm2}                                                                                                                                                                            \\ 
\specialrule{0.10em}{0.5pt}{0.5pt}
Horizon      & \multicolumn{2}{c|}{96}                              & \multicolumn{2}{c|}{192}                             & \multicolumn{2}{c|}{336}                             & \multicolumn{2}{c|}{720}        & \multicolumn{2}{c|}{96}                              & \multicolumn{2}{c|}{192}                             & \multicolumn{2}{c|}{336}                             & \multicolumn{2}{c}{720}         \\ 
\specialrule{0.10em}{0.5pt}{0.5pt}
Metric       & MSE            & \multicolumn{1}{c|}{MAE}            & MSE            & \multicolumn{1}{c|}{MAE}            & MSE            & \multicolumn{1}{c|}{MAE}            & MSE            & MAE            & MSE            & \multicolumn{1}{c|}{MAE}            & MSE            & \multicolumn{1}{c|}{MAE}            & MSE            & \multicolumn{1}{c|}{MAE}            & MSE            & MAE            \\ 
\specialrule{0.10em}{0.5pt}{0.5pt}
DLinear \citeyearpar{DBLP:conf/aaai/ZengCZ023}     & 0.386          & \multicolumn{1}{c|}{0.400}          & 0.437          & \multicolumn{1}{c|}{0.432}          & 0.481          & \multicolumn{1}{c|}{0.459}          & 0.519          & \multicolumn{1}{c|}{0.516}          & 0.193          & \multicolumn{1}{c|}{0.292}          & 0.284          & \multicolumn{1}{c|}{0.362}          & 0.369          & \multicolumn{1}{c|}{0.427}          & 0.554          & 0.522          \\
+PAMod       & \textbf{0.375} & \multicolumn{1}{c|}{\textbf{0.389}} & \textbf{0.423} & \multicolumn{1}{c|}{\textbf{0.425}} & \textbf{0.462} & \multicolumn{1}{c|}{\textbf{0.442}} & \textbf{0.476} & \multicolumn{1}{c|}{\textbf{0.458}} & \textbf{0.177} & \multicolumn{1}{c|}{\textbf{0.255}} & \textbf{0.239} & \multicolumn{1}{c|}{\textbf{0.308}} & \textbf{0.308} & \multicolumn{1}{c|}{\textbf{0.346}} & \textbf{0.422} & \textbf{0.413} \\ 
\specialrule{0.10em}{0.5pt}{0.5pt}
Imp.         & 2.8\%          & \multicolumn{1}{c|}{2.8\%}          & 3.2\%          & \multicolumn{1}{c|}{1.6\%}          & 4.0\%          & \multicolumn{1}{c|}{3.7\%}          & 8.3\%          & \multicolumn{1}{c|}{11.2\%}         & 8.3\%          & \multicolumn{1}{c|}{12.7\%}         & 15.8\%         & \multicolumn{1}{c|}{14.9\%}         & 16.5\%         & \multicolumn{1}{c|}{19.0\%}         & 23.8\%         & 20.9\%           \\ 
\specialrule{0.10em}{0.5pt}{0.5pt}
PatchTST \citeyearpar{DBLP:conf/iclr/NieNSK23}    & 0.414          & \multicolumn{1}{c|}{0.419}          & 0.460          & \multicolumn{1}{c|}{0.445}          & 0.501          & \multicolumn{1}{c|}{0.466}          & 0.500          & \multicolumn{1}{c|}{0.488}          & 0.175          & \multicolumn{1}{c|}{0.259}          & 0.241          & \multicolumn{1}{c|}{0.302}          & 0.305          & \multicolumn{1}{c|}{0.343}          & 0.402          & 0.400          \\
+PAMod       & \textbf{0.396} & \multicolumn{1}{c|}{\textbf{0.392}} & \textbf{0.438} & \multicolumn{1}{c|}{\textbf{0.427}} & \textbf{0.469} & \multicolumn{1}{c|}{\textbf{0.445}} & \textbf{0.473} & \multicolumn{1}{c|}{\textbf{0.461}} & \textbf{0.170} & \multicolumn{1}{c|}{\textbf{0.252}} & \textbf{0.235} & \multicolumn{1}{c|}{\textbf{0.296}} & \textbf{0.299} & \multicolumn{1}{c|}{\textbf{0.337}} & \textbf{0.389} & \textbf{0.388} \\ 
\specialrule{0.10em}{0.5pt}{0.5pt}
Imp.         & 4.3\%          & \multicolumn{1}{c|}{6.4\%}          & 4.8\%          & \multicolumn{1}{c|}{4.0\%}          & 6.4\%          & \multicolumn{1}{c|}{4.5\%}          & 5.4\%          & \multicolumn{1}{c|}{5.5\%}          & 2.9\%          & \multicolumn{1}{c|}{2.7\%}          & 2.5\%          & \multicolumn{1}{c|}{2.0\%}          & 2.0\%          & \multicolumn{1}{c|}{1.7\%}          & 3.2\%          & 3.0\%          \\
\specialrule{0.10em}{0.5pt}{0.5pt}
iTransformer \citeyearpar{DBLP:conf/iclr/LiuHZWWML24} & 0.386          & \multicolumn{1}{c|}{0.405}          & 0.441          & \multicolumn{1}{c|}{0.436}          & 0.487          & \multicolumn{1}{c|}{0.458}          & 0.503          & \multicolumn{1}{c|}{0.491}          & 0.180          & \multicolumn{1}{c|}{0.264}          & 0.250          & \multicolumn{1}{c|}{0.309}          & 0.311          & \multicolumn{1}{c|}{0.348}          & 0.412          & 0.407          \\
+PAMod       & \textbf{0.372} & \multicolumn{1}{c|}{\textbf{0.393}} & \textbf{0.425} & \multicolumn{1}{c|}{\textbf{0.421}} & \textbf{0.464} & \multicolumn{1}{c|}{\textbf{0.443}} & \textbf{0.484} & \multicolumn{1}{c|}{\textbf{0.469}} & \textbf{0.169} & \multicolumn{1}{c|}{\textbf{0.252}} & \textbf{0.238} & \multicolumn{1}{c|}{\textbf{0.292}} & \textbf{0.294} & \multicolumn{1}{c|}{\textbf{0.332}} & \textbf{0.385} & \textbf{0.384} \\ 
\specialrule{0.10em}{0.5pt}{0.5pt}
Imp.         & 3.6\%          & \multicolumn{1}{c|}{3.0\%}          & 3.6\%          & \multicolumn{1}{c|}{3.4\%}          & 4.7\%          & \multicolumn{1}{c|}{3.3\%}          & 3.8\%          & \multicolumn{1}{c|}{4.5\%}          & 6.1\%          & \multicolumn{1}{c|}{4.5\%}          & 4.8\%          & \multicolumn{1}{c|}{5.5\%}          & 5.5\%          & \multicolumn{1}{c|}{4.6\%}          & 6.6\%          & 5.7\%          \\ 
\specialrule{0.10em}{0.5pt}{0.5pt}
TQNet \citeyearpar{DBLP:conf/icml/LinCWQL25}       & 0.371          & \multicolumn{1}{c|}{0.393}          & 0.428          & \multicolumn{1}{c|}{0.426}          & 0.476          & \multicolumn{1}{c|}{0.446}          & 0.487          & \multicolumn{1}{c|}{0.470}          & 0.173          & \multicolumn{1}{c|}{0.256}          & 0.238          & \multicolumn{1}{c|}{0.298}          & 0.301          & \multicolumn{1}{c|}{0.340}          & 0.397          & 0.396          \\
+PAMod       & \textbf{0.369} & \multicolumn{1}{c|}{\textbf{0.387}} & \textbf{0.418} & \multicolumn{1}{c|}{\textbf{0.417}} & \textbf{0.457} & \multicolumn{1}{c|}{\textbf{0.436}} & \textbf{0.453} & \multicolumn{1}{c|}{\textbf{0.455}} & \textbf{0.165} & \multicolumn{1}{c|}{\textbf{0.242}} & \textbf{0.226} & \multicolumn{1}{c|}{\textbf{0.284}} & \textbf{0.285} & \multicolumn{1}{c|}{\textbf{0.323}} & \textbf{0.377} & \textbf{0.380} \\ 
\specialrule{0.10em}{0.5pt}{0.5pt}
Imp.          & 0.5\%          & \multicolumn{1}{c|}{1.5\%}          & 2.3\%          & \multicolumn{1}{c|}{2.1\%}          & 4.0\%          & \multicolumn{1}{c|}{2.2\%}          & 7.0\%          & \multicolumn{1}{c|}{3.2\%}          & 4.6\%          & \multicolumn{1}{c|}{5.5\%}          & 5.0\%          & \multicolumn{1}{c|}{4.7\%}          & 5.3\%          & \multicolumn{1}{c|}{5.6\%}          & 5.0\%          & 4.0\%           \\ 
\specialrule{0.15em}{1pt}{0pt}
\end{tabular}}
\vspace{-0.2cm}
\end{table}

Table \ref{tab:7} shows more compatibility analysis on the ETTh1 and ETTm2 datasets.
Table \ref{tab:8} presents the full comparison results of PAMod against nine baseline methods across 12 real-world multivariate datasets. 
The results demonstrate that PAMod consistently achieves state-of-the-art forecasting performance under most experimental settings, underscoring the effectiveness of the proposed approach.

\subsection{Extra Visual Cases}

\begin{figure}[htbp]
  \centering
  \subfloat[Embedding weights on ETTm1]
  {\includegraphics[width=0.50\textwidth]{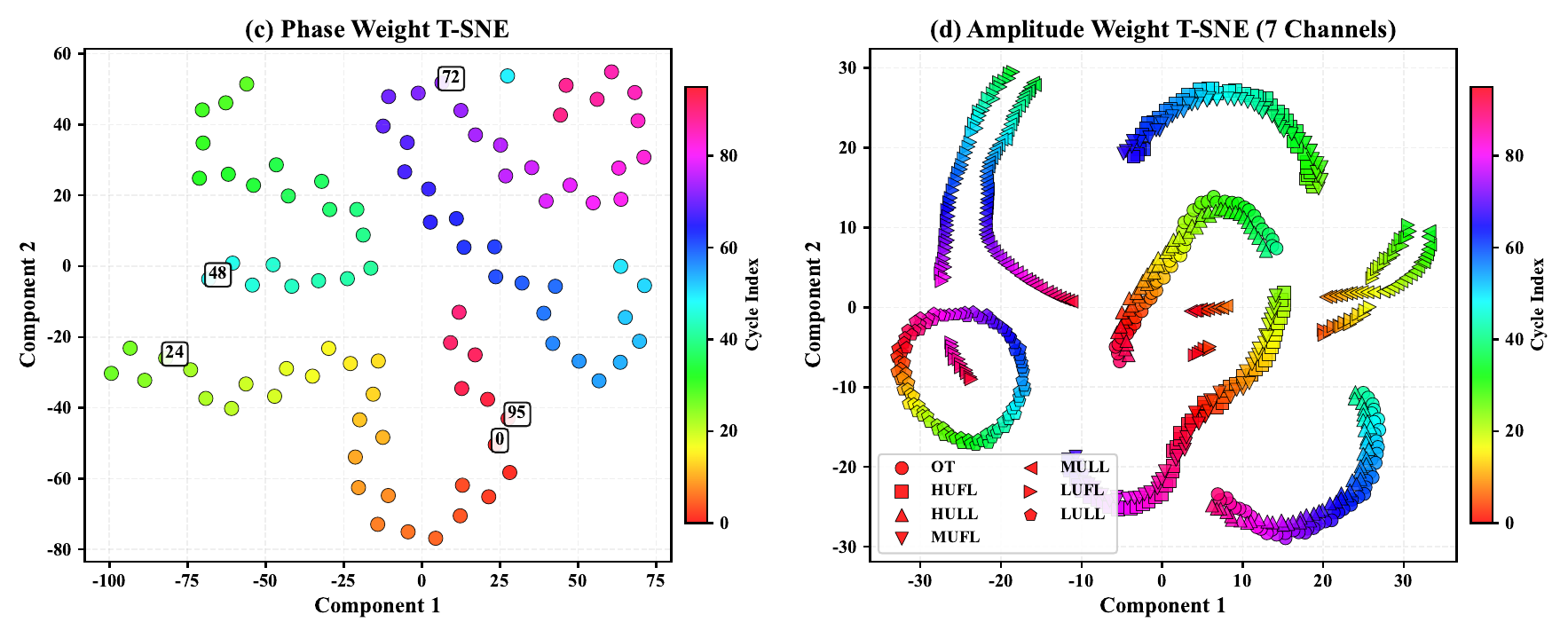}\label{fig:9-1}}
  \hfill
  \subfloat[Distribution alignment on ECL]
  {\includegraphics[width=0.45\textwidth]{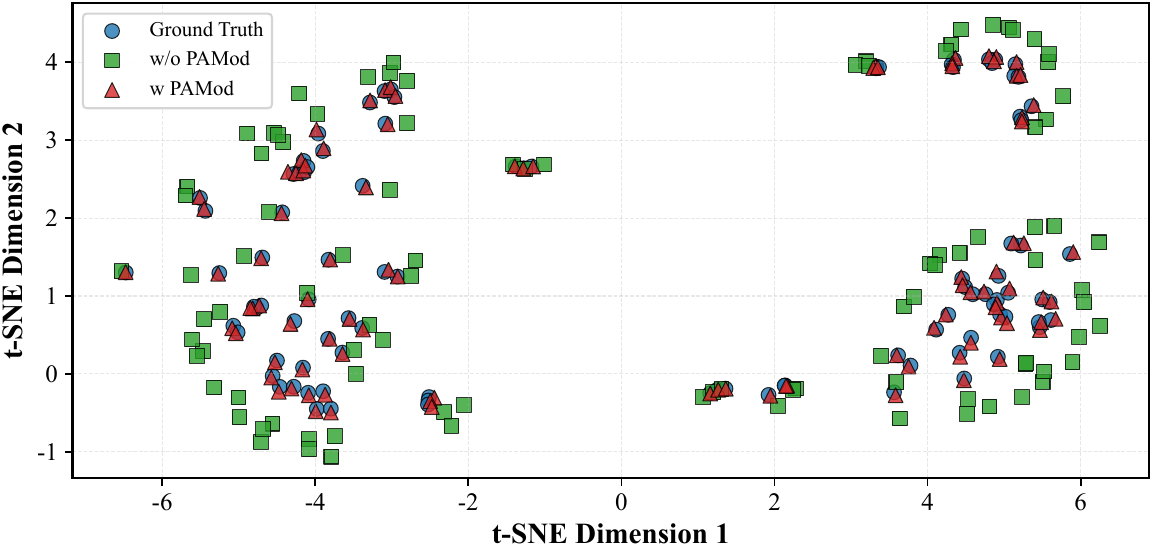}\label{fig:9-2}}\\
  \vspace{0.3cm}
  \subfloat[Forecasting performance on ECL]
  {\includegraphics[width=0.48\textwidth]{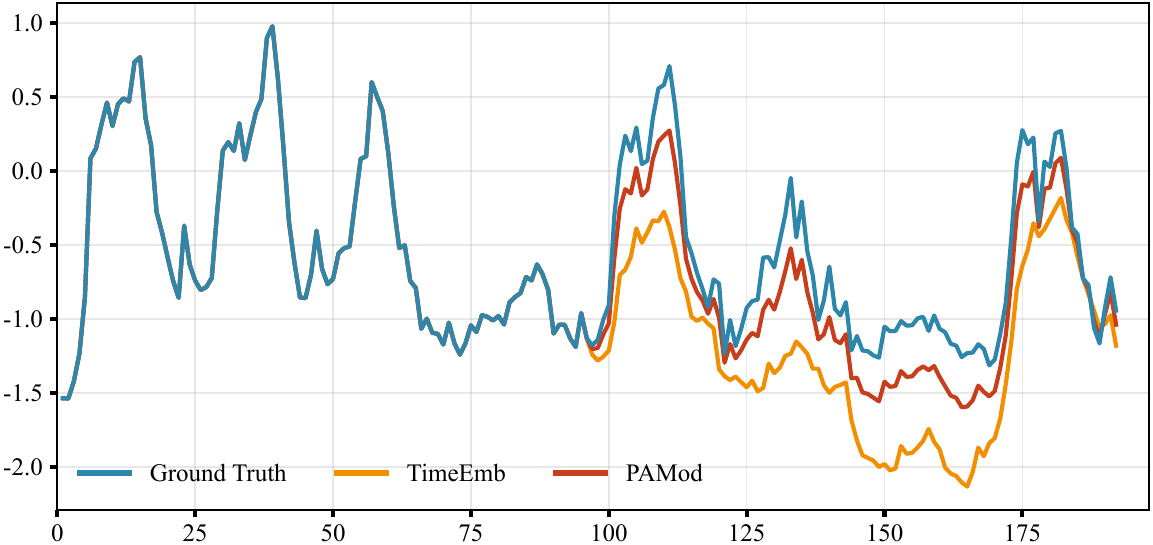}\label{fig:9-3}}
  \hfill
  \subfloat[Lookback sensitivity on Traffic]
  {\includegraphics[width=0.48\textwidth]{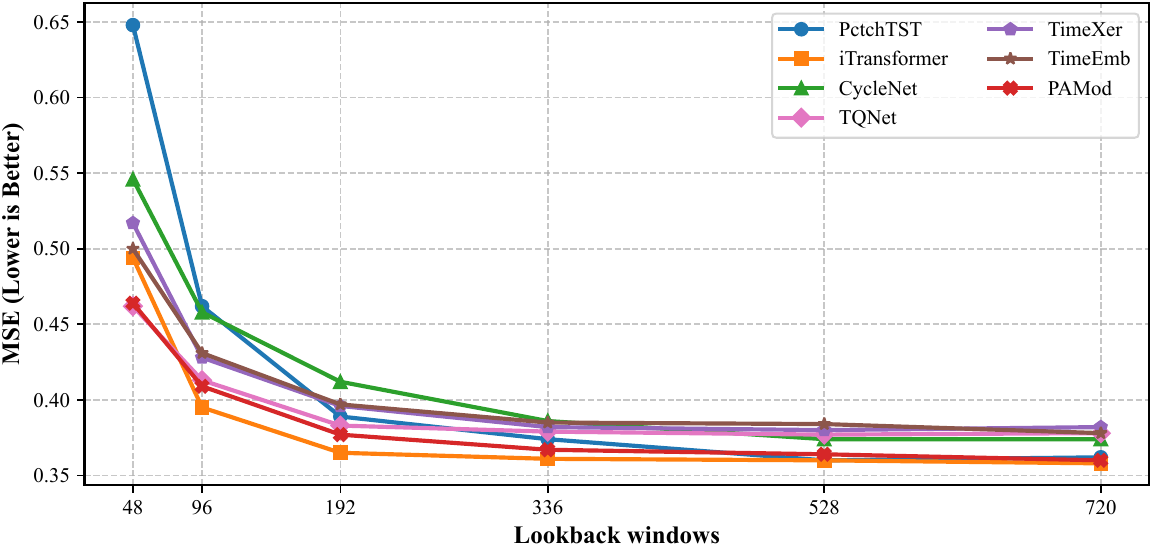}\label{fig:9-4}}
  \caption{Extra visual cases for PAMod}
  \label{fig:9}
\end{figure}

Here, we provide more showcases for visualization in Figure \ref{fig:9}.
In subfigure 9(a), we display the phase and amplitude embedding weights to illustrate their respective roles in modeling mean and variance shifts.
In subfigure 9(b), we show that predictions with PAMod closely overlap the ground truth on ECL, reflecting strong distributional alignment. 
In subfigure 9(c), we exhibit that PAMod achieves a closer fit to the ground truth curve than TimeEmb on ECL, further demonstrating its superior capacity for modeling temporal dynamics. 
In subfigure 9(d), as the lookback window increases, PAMod consistently outperforms most baselines, confirming that modeling cyclical shifts enables effective use of longer historical context.

\begin{table}[htbp]
\centering
\caption{Full results of multivariate time series forecasting across 12 benchmarks. The best results are highlighted in \textbf{bold}, the second best are \underline{underlined}, and the \textbf{Count} row counts the number of times each model ranks in the top 1.}
\label{tab:8}
\scriptsize
\setlength{\tabcolsep}{1.1mm}
{
\begin{tabular}{cc|cc|cc|cc|cc|cc|cc|cc|cc|cc|cc}
\specialrule{0.15em}{0pt}{2pt}
\multicolumn{2}{c|}{\multirow{2}{*}{Model}}         & \multicolumn{2}{c|}{\textbf{PAMod}}  & \multicolumn{2}{c|}{TQNet}   & \multicolumn{2}{c|}{TimeEmb}    & \multicolumn{2}{c|}{FilterTS} & \multicolumn{2}{c|}{Amplifier} & \multicolumn{2}{c|}{TimeXer}    & \multicolumn{2}{c|}{CycleNet} & \multicolumn{2}{c|}{TimeMixer}  & \multicolumn{2}{c|}{iTransformer} & \multicolumn{2}{c}{PatchTST} \\ 
\cmidrule(lr){3-4} \cmidrule(lr){5-6} \cmidrule(lr){7-8} \cmidrule(lr){9-10} \cmidrule(lr){11-12} \cmidrule(lr){13-14} \cmidrule(lr){15-16} \cmidrule(lr){17-18} \cmidrule(lr){19-20} \cmidrule(lr){21-22}
\multicolumn{2}{c|}{}                               & \multicolumn{2}{c|}{\textbf{(Ours)}} & \multicolumn{2}{c|}{\citeyearpar{DBLP:conf/icml/LinCWQL25}}    & \multicolumn{2}{c|}{\citeyearpar{DBLP:journals/corr/abs-2510-00461}}       & \multicolumn{2}{c|}{\citeyearpar{DBLP:conf/aaai/WangLDW25}}     & \multicolumn{2}{c|}{\citeyearpar{DBLP:conf/aaai/Fei000N25}}      & \multicolumn{2}{c|}{\citeyearpar{DBLP:conf/nips/WangWDQZLQWL24}}       & \multicolumn{2}{c|}{\citeyearpar{DBLP:conf/nips/Lin0HWMZ24}}     & \multicolumn{2}{c|}{\citeyearpar{DBLP:conf/iclr/WangWSHLMZ024}}       & \multicolumn{2}{c|}{\citeyearpar{DBLP:conf/iclr/LiuHZWWML24}}         & \multicolumn{2}{c}{\citeyearpar{DBLP:conf/iclr/NieNSK23}}     \\ 
\specialrule{0.10em}{1pt}{1pt}
\multicolumn{2}{c|}{Metric}                         & MSE               & MAE              & MSE            & MAE         & MSE            & MAE            & MSE        & MAE              & MSE          & MAE             & MSE            & MAE            & MSE           & MAE           & MSE            & MAE            & MSE               & MAE           & MSE           & MAE          \\ 
\specialrule{0.10em}{1pt}{1pt}
\multicolumn{1}{c|}{\multirow{5}{*}{ETTm1}}   & 96  & \textbf{0.297}    & \textbf{0.337}   & 0.311          & 0.353       & {\underline{0.305}}    & {\underline{0.344}}    & 0.321      & 0.360            & 0.316        & 0.355           & 0.318          & 0.356          & 0.319         & 0.360         & 0.320          & 0.357          & 0.334             & 0.368         & 0.329         & 0.367        \\
\multicolumn{1}{c|}{}                         & 192 & \textbf{0.346}    & \textbf{0.366}   & 0.356          & 0.378       & {\underline{0.354}}    & {\underline{0.374}}    & 0.363      & 0.382            & 0.361        & 0.381           & 0.362          & 0.383          & 0.360         & 0.381         & 0.361          & 0.381          & 0.377             & 0.391         & 0.367         & 0.385        \\
\multicolumn{1}{c|}{}                         & 336 & \textbf{0.375}    & \textbf{0.387}   & 0.390          & 0.401       & {\underline{0.380}}    & {\underline{0.393}}    & 0.396      & 0.404            & 0.393        & 0.404           & 0.395          & 0.407          & 0.389         & 0.403         & 0.390          & 0.404          & 0.426             & 0.420         & 0.399         & 0.410        \\
\multicolumn{1}{c|}{}                         & 720 & \textbf{0.432}    & \textbf{0.422}   & 0.452          & 0.440       & {\underline{0.434}}    & {\underline{0.427}}    & 0.462      & 0.438            & 0.456        & 0.440           & 0.452          & 0.441          & 0.447         & 0.441         & 0.454          & 0.441          & 0.491             & 0.459         & 0.454         & 0.439        \\ 
\cmidrule(lr){2-22}
\multicolumn{1}{c|}{}                         & Avg & \textbf{0.363}    & \textbf{0.378}   & 0.377          & 0.393       & {\underline{0.368}}    & {\underline{0.385}}    & 0.385      & 0.396            & 0.382        & 0.395           & 0.382          & 0.397          & 0.379         & 0.396         & 0.381          & 0.395          & 0.407             & 0.410         & 0.387         & 0.400        \\ 
\specialrule{0.10em}{1pt}{1pt}
\multicolumn{1}{c|}{\multirow{5}{*}{ETTm2}}   & 96  & \textbf{0.162}    & \textbf{0.240}   & 0.173          & 0.256       & 0.164          & {\underline{0.243}}    & 0.173      & 0.256            & 0.178        & 0.261           & 0.171          & 0.256          & {\underline{0.163}}   & 0.246         & 0.175          & 0.258          & 0.180             & 0.264         & 0.175         & 0.259        \\
\multicolumn{1}{c|}{}                         & 192 & \textbf{0.226}    & \textbf{0.283}   & 0.238          & 0.298       & {\underline{0.227}}    & {\underline{0.285}}    & 0.238      & 0.299            & 0.244        & 0.304           & 0.237          & 0.299          & 0.229         & 0.290         & 0.237          & 0.299          & 0.250             & 0.309         & 0.241         & 0.302        \\
\multicolumn{1}{c|}{}                         & 336 & \textbf{0.283}    & \textbf{0.321}   & 0.301          & 0.340       & 0.285          & {\underline{0.323}}    & 0.300      & 0.338            & 0.309        & 0.346           & 0.296          & 0.338          & {\underline{0.284}}   & 0.327         & 0.298          & 0.340          & 0.311             & 0.348         & 0.305         & 0.343        \\
\multicolumn{1}{c|}{}                         & 720 & \textbf{0.382}    & \textbf{0.382}   & 0.397          & 0.396       & {\underline{0.383}}    & \textbf{0.382} & 0.399      & 0.395            & 0.390        & 0.394           & 0.392          & 0.394          & 0.389         & {\underline{0.391}}   & 0.391          & 0.396          & 0.412             & 0.407         & 0.402         & 0.400        \\ 
\cmidrule(lr){2-22}
\multicolumn{1}{c|}{}                         & Avg & \textbf{0.263}    & \textbf{0.307}   & 0.277          & 0.323       & {\underline{0.265}}    & {\underline{0.308}}    & 0.277      & 0.322            & 0.280        & 0.326           & 0.274          & 0.322          & 0.266         & 0.314         & 0.275          & 0.323          & 0.288             & 0.332         & 0.281         & 0.326        \\ 
\specialrule{0.10em}{1pt}{1pt}
\multicolumn{1}{c|}{\multirow{5}{*}{ETTh1}}   & 96  & \textbf{0.357}    & \textbf{0.382}   & 0.371          & 0.393       & {\underline{0.366}}    & {\underline{0.387}}    & 0.375      & 0.391            & 0.371        & 0.392           & 0.382          & 0.403          & 0.375         & 0.395         & 0.375          & 0.400          & 0.386             & 0.405         & 0.414         & 0.419        \\
\multicolumn{1}{c|}{}                         & 192 & \textbf{0.403}    & \textbf{0.413}   & 0.428          & 0.426       & {\underline{0.417}}    & {\underline{0.416}}    & 0.424      & 0.421            & 0.425        & 0.422           & 0.429          & 0.435          & 0.436         & 0.428         & 0.429          & 0.421          & 0.441             & 0.436         & 0.460         & 0.445        \\
\multicolumn{1}{c|}{}                         & 336 & \textbf{0.443}    & \textbf{0.434}   & 0.476          & 0.446       & 0.457          & {\underline{0.436}}    & 0.465      & 0.442            & {\underline{0.448}}  & \textbf{0.434}  & 0.468          & 0.448          & 0.496         & 0.455         & 0.484          & 0.458          & 0.487             & 0.458         & 0.501         & 0.466        \\
\multicolumn{1}{c|}{}                         & 720 & \textbf{0.447}    & \textbf{0.452}   & 0.487          & 0.470       & {\underline{0.459}}    & {\underline{0.460}}    & 0.472      & 0.466            & 0.476        & 0.464           & 0.469          & 0.461          & 0.520         & 0.484         & 0.498          & 0.482          & 0.503             & 0.491         & 0.500         & 0.488        \\ 
\cmidrule(lr){2-22}
\multicolumn{1}{c|}{}                         & Avg & \textbf{0.413}    & \textbf{0.420}   & 0.441          & 0.434       & {\underline{0.425}}    & {\underline{0.425}}    & 0.434      & 0.430            & 0.430        & 0.428           & 0.437          & 0.437          & 0.457         & 0.441         & 0.447          & 0.440          & 0.454             & 0.448         & 0.469         & 0.455        \\ 
\specialrule{0.10em}{1pt}{1pt}
\multicolumn{1}{c|}{\multirow{5}{*}{ETTh2}}   & 96  & {\underline{0.279}}       & \textbf{0.328}   & 0.295          & 0.343       & \textbf{0.277} & \textbf{0.328} & 0.290      & {\underline{0.338}}      & 0.290        & 0.341           & 0.286          & {\underline{0.338}}    & 0.298         & 0.344         & 0.289          & 0.341          & 0.297             & 0.349         & 0.302         & 0.348        \\
\multicolumn{1}{c|}{}                         & 192 & \textbf{0.349}    & \textbf{0.374}   & 0.367          & 0.393       & {\underline{0.354}}    & {\underline{0.378}}    & 0.374      & 0.390            & 0.369        & 0.390           & 0.363          & 0.389          & 0.372         & 0.396         & 0.372          & 0.392          & 0.380             & 0.400         & 0.388         & 0.400        \\
\multicolumn{1}{c|}{}                         & 336 & 0.403             & \textbf{0.414}   & 0.417          & 0.427       & {\underline{0.400}}    & {\underline{0.417}}    & 0.415      & 0.424            & 0.419        & 0.431           & 0.414          & 0.423          & 0.431         & 0.439         & \textbf{0.386} & \textbf{0.414} & 0.428             & 0.432         & 0.426         & 0.433        \\
\multicolumn{1}{c|}{}                         & 720 & {\underline{0.409}}       & \textbf{0.432}   & 0.433          & 0.446       & 0.416          & 0.437          & 0.420      & 0.438            & 0.446        & 0.456           & \textbf{0.408} & \textbf{0.432} & 0.450         & 0.458         & 0.412          & {\underline{0.434}}    & 0.427             & 0.445         & 0.431         & 0.446        \\ 
\cmidrule(lr){2-22}
\multicolumn{1}{c|}{}                         & Avg & \textbf{0.360}    & \textbf{0.387}   & 0.378          & 0.402       & {\underline{0.362}}    & {\underline{0.390}}    & 0.375      & 0.398            & 0.381        & 0.405           & 0.368          & 0.396          & 0.388         & 0.409         & 0.364          & 0.395          & 0.383             & 0.407         & 0.387         & 0.407        \\ 
\specialrule{0.10em}{1pt}{1pt}
\multicolumn{1}{c|}{\multirow{5}{*}{ECL}}     & 96  & {\underline{0.136}}       & \textbf{0.226}   & \textbf{0.134} & {\underline{0.229}} & {\underline{0.137}}    & {\underline{0.232}}    & 0.151      & 0.245            & 0.147        & 0.242           & 0.140          & 0.242          & {\underline{0.136}}   & {\underline{0.229}}   & 0.153          & 0.247          & 0.148             & 0.240         & 0.181         & 0.270        \\
\multicolumn{1}{c|}{}                         & 192 & \textbf{0.151}    & \textbf{0.240}   & 0.154          & 0.247       & 0.154          & 0.248          & 0.164      & 0.256            & 0.158        & 0.251           & 0.157          & 0.256          & {\underline{0.152}}   & {\underline{0.244}}   & 0.166          & 0.256          & 0.162             & 0.253         & 0.188         & 0.274        \\
\multicolumn{1}{c|}{}                         & 336 & \textbf{0.168}    & \textbf{0.258}   & {\underline{0.169}}    & {\underline{0.264}} & 0.171          & 0.265          & 0.181      & 0.274            & 0.175        & 0.271           & 0.176          & 0.275          & 0.170         & {\underline{0.264}}   & 0.185          & 0.277          & 0.178             & 0.269         & 0.204         & 0.293        \\
\multicolumn{1}{c|}{}                         & 720 & {\underline{0.205}}       & \textbf{0.290}   & \textbf{0.201} & {\underline{0.294}} & {\underline{0.209}}    & {\underline{0.298}}    & 0.225      & 0.311            & 0.206        & 0.298           & 0.211          & 0.306          & 0.212         & 0.299         & 0.225          & 0.310          & 0.225             & 0.317         & 0.246         & 0.324        \\ 
\cmidrule(lr){2-22}
\multicolumn{1}{c|}{}                         & Avg & {\underline{0.165}}       & \textbf{0.254}   & \textbf{0.164} & {\underline{0.259}} & 0.168          & 0.261          & 0.180      & 0.272            & 0.172        & 0.266           & 0.171          & 0.270          & 0.168         & 0.259         & 0.182          & 0.272          & 0.178             & 0.270         & 0.205         & 0.290        \\ 
\specialrule{0.10em}{1pt}{1pt}
\multicolumn{1}{c|}{\multirow{5}{*}{Traffic}} & 96  & {\underline{0.409}}       & \textbf{0.256}   & 0.413          & {\underline{0.261}} & 0.431          & 0.280          & 0.446      & 0.308            & 0.456        & 0.299           & 0.428          & 0.271          & 0.458         & 0.296         & 0.462          & 0.285          & \textbf{0.395}    & 0.268         & 0.462         & 0.290        \\
\multicolumn{1}{c|}{}                         & 192 & {\underline{0.424}}       & \textbf{0.268}   & 0.432          & {\underline{0.271}} & 0.442          & 0.291          & 0.456      & 0.308            & 0.472        & 0.318           & 0.448          & 0.282          & 0.457         & 0.294         & 0.473          & 0.296          & \textbf{0.417}    & 0.276         & 0.466         & 0.290        \\
\multicolumn{1}{c|}{}                         & 336 & {\underline{0.436}}       & \textbf{0.274}   & 0.450          & {\underline{0.277}} & 0.455          & 0.297          & 0.472      & 0.313            & 0.487        & 0.320           & 0.473          & 0.289          & 0.470         & 0.299         & 0.498          & 0.296          & \textbf{0.433}    & 0.283         & 0.482         & 0.300        \\
\multicolumn{1}{c|}{}                         & 720 & \textbf{0.465}    & \textbf{0.291}   & 0.486          & {\underline{0.295}} & 0.484          & 0.313          & 0.508      & 0.332            & 0.517        & 0.332           & 0.516          & 0.307          & 0.502         & 0.314         & 0.506          & 0.313          & {\underline{0.467}}       & 0.302         & 0.514         & 0.320        \\ 
\cmidrule(lr){2-22}
\multicolumn{1}{c|}{}                         & Avg & {\underline{0.434}}       & \textbf{0.272}   & 0.445          & {\underline{0.276}} & 0.453          & 0.295          & 0.470      & 0.315            & 0.483        & 0.317           & 0.466          & 0.287          & 0.472         & 0.314         & 0.484          & 0.297          & \textbf{0.428}    & 0.282         & 0.481         & 0.300        \\ 
\specialrule{0.10em}{1pt}{1pt}
\multicolumn{1}{c|}{\multirow{5}{*}{Weather}} & 96  & {\underline{0.153}}       & {\underline{0.191}}      & 0.157          & 0.200       & \textbf{0.150} & \textbf{0.190} & 0.162      & 0.207            & 0.167        & 0.212           & 0.157          & 0.205          & 0.158         & 0.203         & 0.163          & 0.209          & 0.174             & 0.214         & 0.177         & 0.210        \\
\multicolumn{1}{c|}{}                         & 192 & \textbf{0.201}    & \textbf{0.237}   & 0.206          & 0.245       & \textbf{0.201} & {\underline{0.238}}    & 0.209      & 0.252            & 0.215        & 0.251           & {\underline{0.204}}          & 0.247          & 0.207         & 0.247         & 0.208          & 0.250          & 0.221             & 0.254         & 0.225         & 0.250        \\
\multicolumn{1}{c|}{}                         & 336 & \textbf{0.259}    & \textbf{0.281}   & 0.262          & 0.287       & \textbf{0.259} & {\underline{0.282}}    & 0.263      & 0.292            & 0.276        & 0.292           & 0.261          & 0.290          & 0.262         & 0.289         & 0.251          & 0.287          & 0.278             & 0.296         & 0.278         & 0.290        \\
\multicolumn{1}{c|}{}                         & 720 & {0.341}       & \textbf{0.336}   & 0.344          & 0.342       & \textbf{0.339} & \textbf{0.336} & 0.345      & 0.344            & 0.352        & 0.346           & {\underline{0.340}}          & 0.341          & 0.344         & 0.344         & 0.339          & 0.341          & 0.358             & 0.349         & 0.354         & 0.340        \\ 
\cmidrule(lr){2-22}
\multicolumn{1}{c|}{}                         & Avg & {\underline{0.239}}       & \textbf{0.261}   & 0.242          & 0.269       & \textbf{0.237} & {\underline{0.262}}    & 0.245      & 0.274            & 0.253        & 0.275           & 0.241          & 0.271          & 0.243         & 0.271         & 0.240          & 0.271          & 0.258             & 0.278         & 0.259         & 0.273        \\ 
\specialrule{0.10em}{1pt}{1pt}
\multicolumn{1}{c|}{\multirow{5}{*}{Solar}}   & 96  & {\underline{0.184}}       & \textbf{0.209}   & \textbf{0.173} & {\underline{0.233}} & 0.206          & 0.241          & 0.196      & 0.264            & 0.234        & 0.283           & 0.215          & 0.295          & 0.190         & 0.247         & 0.189          & 0.259          & 0.203             & 0.237         & 0.234         & 0.286        \\
\multicolumn{1}{c|}{}                         & 192 & {\underline{0.205}}       & \textbf{0.226}   & \textbf{0.199} & {\underline{0.257}} & 0.241          & 0.262          & 0.211      & 0.278            & 0.237        & 0.259           & 0.236          & 0.301          & 0.210         & 0.266         & 0.222          & 0.283          & 0.233             & 0.261         & 0.267         & 0.310        \\
\multicolumn{1}{c|}{}                         & 336 & {\underline{0.215}}       & \textbf{0.232}   & \textbf{0.211} & {\underline{0.263}} & 0.268          & 0.284          & 0.226      & 0.284            & 0.247        & 0.269           & 0.252          & 0.307          & 0.217         & 0.266         & 0.231          & 0.292          & 0.248             & 0.273         & 0.290         & 0.315        \\
\multicolumn{1}{c|}{}                         & 720 & {\underline{0.216}}       & \textbf{0.238}   & \textbf{0.209} & 0.270       & 0.278          & 0.291          & 0.227      & 0.282            & 0.246        & 0.270           & 0.244          & 0.305          & 0.223         & {\underline{0.266}}   & 0.223          & 0.285          & 0.249             & 0.275         & 0.289         & 0.317        \\ 
\cmidrule(lr){2-22}
\multicolumn{1}{c|}{}                         & Avg & {\underline{0.206}}       & \textbf{0.226}   & \textbf{0.198} & {\underline{0.256}} & 0.248          & 0.270          & 0.215      & 0.277            & 0.241        & 0.270           & 0.237          & 0.302          & 0.210         & 0.261         & 0.216          & 0.280          & 0.233             & 0.262         & 0.270         & 0.307        \\ 
\specialrule{0.10em}{1pt}{1pt}
\multicolumn{1}{c|}{\multirow{5}{*}{PEMS03}}  & 12  & \textbf{0.059}    & \textbf{0.157}   & {\underline{0.060}}    & {\underline{0.161}} & 0.066          & 0.168          & 0.072      & 0.181            & 0.070        & 0.172           & 0.070          & 0.173          & 0.066         & 0.172         & 0.076          & 0.188          & 0.071             & 0.174         & 0.099         & 0.216        \\
\multicolumn{1}{c|}{}                         & 24  & \textbf{0.072}    & \textbf{0.173}   & {\underline{0.077}}    & {\underline{0.182}} & 0.083          & 0.189          & 0.104      & 0.219            & 0.090        & 0.200           & 0.092          & 0.194          & 0.089         & 0.201         & 0.113          & 0.226          & 0.093             & 0.201         & 0.142         & 0.259        \\
\multicolumn{1}{c|}{}                         & 48  & \textbf{0.098}    & \textbf{0.199}   & {\underline{0.104}}    & {\underline{0.215}} & 0.116          & 0.220          & 0.155      & 0.269            & 0.147        & 0.260           & 0.129          & 0.229          & 0.136         & 0.247         & 0.191          & 0.292          & 0.125             & 0.236         & 0.211         & 0.319        \\
\multicolumn{1}{c|}{}                         & 96  & \textbf{0.133}    & \textbf{0.228}   & {\underline{0.148}}    & {\underline{0.253}} & 0.151          & 0.251          & 0.203      & 0.315            & 0.217        & 0.323           & 0.157          & 0.261          & 0.182         & 0.282         & 0.288          & 0.363          & 0.164             & 0.275         & 0.269         & 0.370        \\ 
\cmidrule(lr){2-22}
\multicolumn{1}{c|}{}                         & Avg & \textbf{0.091}    & \textbf{0.189}   & {\underline{0.097}}    & {\underline{0.203}} & 0.104          & 0.207          & 0.134      & 0.246            & 0.131        & 0.239           & 0.112          & 0.214          & 0.118         & 0.226         & 0.167          & 0.267          & 0.113             & 0.222         & 0.180         & 0.291        \\ 
\specialrule{0.10em}{1pt}{1pt}
\multicolumn{1}{c|}{\multirow{5}{*}{PEMS04}}  & 12  & \textbf{0.066}    & \textbf{0.162}   & {\underline{0.067}}    & {\underline{0.166}} & 0.071          & 0.172          & 0.087      & 0.199            & 0.082        & 0.190           & 0.074          & 0.178          & 0.078         & 0.186         & 0.092          & 0.204          & 0.078             & 0.183         & 0.105         & 0.224        \\
\multicolumn{1}{c|}{}                         & 24  & \textbf{0.075}    & \textbf{0.174}   & {\underline{0.077}}    & {\underline{0.181}} & 0.083          & 0.187          & 0.107      & 0.223            & 0.102        & 0.215           & 0.087          & 0.195          & 0.099         & 0.212         & 0.128          & 0.243          & 0.095             & 0.205         & 0.153         & 0.275        \\
\multicolumn{1}{c|}{}                         & 48  & \textbf{0.091}    & \textbf{0.192}   & {\underline{0.097}}    & {\underline{0.206}} & 0.104          & 0.210          & 0.138      & 0.255            & 0.151        & 0.269           & 0.110          & 0.214          & 0.133         & 0.248         & 0.213          & 0.315          & 0.120             & 0.233         & 0.229         & 0.339        \\
\multicolumn{1}{c|}{}                         & 96  & \textbf{0.111}    & \textbf{0.212}   & {\underline{0.123}}    & {\underline{0.233}} & 0.126          & 0.231          & 0.166      & 0.285            & 0.205        & 0.320           & 0.148          & 0.251          & 0.167         & 0.281         & 0.307          & 0.384          & 0.150             & 0.262         & 0.291         & 0.389        \\ 
\cmidrule(lr){2-22}
\multicolumn{1}{c|}{}                         & Avg & \textbf{0.086}    & \textbf{0.185}   & {\underline{0.091}}    & {\underline{0.197}} & 0.096          & 0.200          & 0.125      & 0.241            & 0.135        & 0.249           & 0.105          & 0.209          & 0.119         & 0.232         & 0.185          & 0.287          & 0.111             & 0.221         & 0.195         & 0.307        \\ 
\specialrule{0.10em}{1pt}{1pt}
\multicolumn{1}{c|}{\multirow{5}{*}{PEMS07}}  & 12  & {\underline{0.053}}       & \textbf{0.141}   & \textbf{0.051} & {\underline{0.143}} & 0.057          & 0.151          & 0.067      & 0.170            & 0.079        & 0.179           & 0.057          & 0.152          & 0.062         & 0.162         & 0.073          & 0.184          & 0.067             & 0.265         & 0.095         & 0.207        \\
\multicolumn{1}{c|}{}                         & 24  & {\underline{0.065}}       & \textbf{0.156}   & \textbf{0.063} & {\underline{0.159}} & 0.075          & 0.171          & 0.095      & 0.203            & 0.094        & 0.196           & 0.079          & 0.179          & 0.086         & 0.192         & 0.111          & 0.219          & 0.088             & 0.190         & 0.150         & 0.262        \\
\multicolumn{1}{c|}{}                         & 48  & {\underline{0.086}}       & \textbf{0.176}   & \textbf{0.081} & {\underline{0.179}} & 0.106          & 0.200          & 0.146      & 0.239            & 0.129        & 0.237           & 0.099          & 0.191          & 0.128         & 0.234         & 0.237          & 0.328          & 0.110             & 0.215         & 0.253         & 0.340        \\
\multicolumn{1}{c|}{}                         & 96  & {\underline{0.116}}       & \textbf{0.199}   & \textbf{0.103} & {\underline{0.203}} & 0.148          & 0.231          & 0.172      & 0.266            & 0.185        & 0.291           & 0.107          & 0.205          & 0.176         & 0.268         & 0.303          & 0.354          & 0.139             & 0.245         & 0.346         & 0.404        \\ 
\cmidrule(lr){2-22}
\multicolumn{1}{c|}{}                         & Avg & {\underline{0.080}}       & \textbf{0.168}   & \textbf{0.075} & {\underline{0.171}} & 0.097          & 0.188          & 0.120      & 0.220            & 0.122        & 0.226           & 0.085          & 0.182          & 0.113         & 0.214         & 0.181          & 0.271          & 0.101             & 0.204         & 0.211         & 0.303        \\ 
\specialrule{0.10em}{1pt}{1pt}
\multicolumn{1}{c|}{\multirow{5}{*}{PEMS08}}  & 12  & {\underline{0.074}}       & \textbf{0.161}   & \textbf{0.071} & {\underline{0.170}} & 0.087          & 0.188          & 0.086      & 0.194            & 0.079        & 0.182           & 0.075          & 0.176          & 0.082         & 0.185         & 0.091          & 0.201          & 0.079             & 0.182         & 0.168         & 0.232        \\
\multicolumn{1}{c|}{}                         & 24  & \textbf{0.094}    & \textbf{0.177}   & {\underline{0.096}}    & {\underline{0.196}} & 0.106          & 0.198          & 0.126      & 0.236            & 0.115        & 0.218           & 0.102          & 0.201          & 0.117         & 0.226         & 0.137          & 0.246          & 0.115             & 0.219         & 0.224         & 0.281        \\
\multicolumn{1}{c|}{}                         & 48  & \textbf{0.126}    & \textbf{0.199}   & {\underline{0.149}}    & 0.244       & 0.158          & 0.253          & 0.223      & 0.320            & 0.192        & 0.294           & 0.158          & 0.248          & 0.169         & 0.268         & 0.265          & 0.343          & 0.186             & {\underline{0.235}}   & 0.321         & 0.354        \\
\multicolumn{1}{c|}{}                         & 96  & {\underline{0.181}}       & \textbf{0.223}   & 0.253          & 0.309       & \textbf{0.177} & {\underline{0.246}}    & 0.284      & 0.315            & 0.346        & 0.390           & 0.366          & 0.377          & 0.233         & 0.306         & 0.410          & 0.407          & 0.221             & 0.267         & 0.408         & 0.417        \\ 
\cmidrule(lr){2-22}
\multicolumn{1}{c|}{}                         & Avg & \textbf{0.119}    & \textbf{0.190}   & 0.142          & 0.229       & {\underline{0.132}}    & {\underline{0.221}}    & 0.180      & 0.266            & 0.183        & 0.271           & 0.175          & 0.250          & 0.150         & 0.246         & 0.226          & 0.299          & 0.150             & 0.226         & 0.280         & 0.321        \\ 
\specialrule{0.10em}{1pt}{1pt}
\multicolumn{2}{c|}{$1^{st}$ Count}                 & \textbf{35}       & \textbf{59}      & {\underline{14}}       & 0           & 7              & {\underline{4}}        & 0          & 0                & 0            & 1               & 1              & 1              & 0             & 0             & 1              & 1              & 4                 & 0             & 0             & 0            \\ 
\specialrule{0.15em}{2pt}{0pt}
\end{tabular}}
\end{table}

\section{Discussion}

\textbf{Potential limitations}. 
While PAMod demonstrates strong performance in forecasting tasks with cyclical distribution shifts, several limitations merit discussion:
\begin{itemize}
    \item \textbf{Fixed cycle assumption}. PAMod relies on a predefined cycle length $\tau$, which may not adapt well to series with varying or unknown periodicities (e.g., irregular event-driven patterns or multi-frequency cycles without a dominant period).
    \item \textbf{Channel-common cycle}. The current design applies the same cycle length to all channels, which could be suboptimal when different variables exhibit distinct periodic behaviors (e.g., temperature with daily cycles and electricity load with weekly cycles). A channel-adaptive cycle mechanism could offer a more tailored solution.
    \item \textbf{Sensitivity to cycle misalignment}. If the assumed cycle length significantly deviates from the true data periodicity, the phase and amplitude embeddings may fail to capture the actual distribution structure, leading to reduced robustness in shift modeling.
    \item \textbf{Limited to cyclical shifts}. PAMod is designed to model cyclically structured distribution shifts, but may not effectively handle abrupt, non-cyclical regime changes or anomalies that fall outside learned periodic patterns.
\end{itemize}

\textbf{Future Directions}. 
Several directions remain open for extending PAMod’s capability and applicability:
\begin{itemize}
    \item \textbf{Adaptive cycle learning}. Instead of using a fixed $\tau$, future versions could learn or dynamically adjust the cycle length from data, e.g., via periodicity estimation modules or multi-scale cycle banks.
    \item \textbf{Channel-wise cycle modeling}. Incorporating variable-specific cycle embeddings would allow PAMod to better handle multivariate series where different channels follow distinct periodic regimes.
    \item \textbf{Integration with structural dependencies}. Although PAMod focuses on temporal cyclicity, combining it with inter-variable relationship modeling (e.g., graph-based or attention-based channel interaction) could further improve performance in spatio-temporal or highly correlated multivariate settings.
    \item \textbf{Generalization to other non-stationary tasks}. Exploring PAMod’s adaptation to other domains with structured distribution shifts—such as financial volatility, energy load forecasting with weather shocks, or healthcare monitoring with seasonal effects—would test its broader utility.
\end{itemize}

In summary, while PAMod provides an effective framework for cyclical distribution modeling, further advances in adaptive cycle inference, channel-aware design, and integration with relational modeling could enhance its flexibility and scope.


\end{document}